\documentclass[preprint,12pt,times,3p]{elsarticle}
\usepackage{minitoc}
\usepackage[dvipsnames]{xcolor}
\usepackage{natbib}
\usepackage{subcaption}
\usepackage{bbm, dsfont} % indicator function
\usepackage{amssymb}
\usepackage{amsthm}
\usepackage{stmaryrd} % llbracket
\newcommand{\ours}{PIPE\,}
\usepackage{times}
\usepackage{epsfig}
\usepackage{graphicx}
\usepackage{amsmath}
\usepackage{bbm, dsfont}
\usepackage{amssymb}
\usepackage{booktabs}
\definecolor{Gray}{gray}{0.85}
\usepackage{multirow}
\usepackage{pifont}% http://ctan.org/pkg/pifont
\newcommand{\cmark}{\textcolor{OliveGreen}{\ding{51}}}%
\newcommand{\xmark}{\textcolor{red}{\ding{55}}}%
\usepackage{algorithm}
\usepackage{algpseudocode}
\usepackage{comment}
\usepackage{hyperref}
\usepackage{caption}
\usepackage{tikz}
\usepackage{comment}
\usepackage{amsmath,amssymb} % define this before the line numbering.
\usepackage{color}
\definecolor{Gray}{gray}{0.85}
\usepackage{array}
\usepackage{eqparbox}
\renewcommand\algorithmiccomment[1]{%
  \hfill\textcolor{OliveGreen}{$\blacktriangleright$}\ \eqparbox{COMMENT}{#1}%
}
\usepackage[capitalize,noabbrev]{cleveref}

\newtheorem{theorem}{Theorem}[section]

\newtheorem{lemma}[theorem]{Lemma}
\newtheorem{corollary}[theorem]{Corollary}
\usepackage[textsize=tiny]{todonotes}
\usepackage{amssymb}
\journal{Arxiv}

\begin{document}

\begin{frontmatter}



\title{\ours: Parallelized Inference Through Post-Training Quantization Ensembling of Residual Expansions}


\author[label1,label2]{Edouard Yvinec}
\author[label1]{Arnaud Dapogny}
\author[label1,label2]{Kevin Bailly}

\affiliation[label1]{organization={Datakalab},
            addressline={143 Avenue Charles de Gaulle}, 
            city={Neuilly-sur-Seine},
            postcode={92200}, 
            state={Ile-de-France},
            country={France}}

\affiliation[label2]{organization={Sorbonne Université, CNRS, ISIR},
            addressline={4 Place Jussieu}, 
            city={Paris},
            postcode={f-75005}, 
            state={Ile-de-France},
            country={France}}
\begin{abstract}
Deep neural networks (DNNs) are ubiquitous in computer vision and natural language processing, but suffer from high inference cost. This problem can be addressed by quantization, which consists in converting floating point operations into a lower bit-width format. With the growing concerns on privacy rights, we focus our efforts on data-free methods. However, such techniques suffer from their lack of adaptability to the target devices, as a hardware typically only support specific bit widths. Thus, to adapt to a variety of devices, a quantization method shall be flexible enough to find good accuracy \textit{v.s.} speed trade-offs for every bit width and target device. To achieve this, we propose \ours, a quantization method that leverages residual error expansion, along with group sparsity and an ensemble approximation for better parallelization. \ours is backed off by strong theoretical guarantees and achieves superior performance on every benchmarked application (from vision to NLP tasks), architecture (ConvNets, transformers) and bit-width (from int8 to ternary quantization).
\end{abstract}


% \begin{highlights}
% \item Deep neural networks benefits from quantization for inference
% \item Quantization techniques offer few trade-offs between speed and accuracy given a target hardware
% \item Residual expansions of quantized network enable multiple speed and accuracy trade-offs
% \item PIPE derives multiple networks from a single expanded network
% \item PIPE leverages ensembles of quantized neural networks
% \end{highlights}

\begin{keyword}


Quantization \sep Deep Learning \sep LLM \sep Ensemble \sep Efficient Inference
\end{keyword}

\end{frontmatter}

\section{Introduction}
Deep neural networks (DNNs) achieve outstanding performance on several challenging computer vision tasks such as image classification \cite{he2016deep}, object detection \cite{liu2016ssd} and semantic segmentation \cite{chen2017rethinking}, as well as natural language processing benchmarks such as text classification \cite{devlin2018bert}. However, their accuracy comes at a high computational inference cost which limits their deployment, more so on edge devices when real-time treatment as well as energy consumption are a concern. This problem can be tackled \textit{via} DNN quantization, \textit{i.e.} by reducing the bit-width representation of the computations from floating point operations (FP) to e.g. int8 (8-bits integer representation), int4, int3 or even lower-bit representation such as ternary (where weights values are either $-1$, $0$ or $+1$) quantization. Because DNN inference principally relies on matrix multiplication, such quantization dramatically diminishes the number of bit-wise operations (as defined in \cite{krishnamoorthi2018quantizing}), thus limiting the DNN latency and energy consumption. However, DNN quantization usually comes at the expense of the network accuracy. As a consequence, DNN quantization is an active field of research \cite{courbariaux2016binarized,wu2018training,jacob2018quantization,achterhold2018variational,louizos2018relaxed,sheng2018quantization,choi2022s,zhong2022intraq} that aims at limiting this accuracy drop while reducing the number of bit-wise operations.

All the aforementioned methods are data-driven, as they either involve training a network from scratch or fine-tune an already trained and quantized one. However, while such approaches usually allow lower quantization errors using low bit-wise representations, due to the growing concerns on privacy rights and data privacy, there is an ever-increasing number of real-case scenarios (e.g. health and military services) where data may not be available for quantization purposes. Motivated by these observations, recently, several data-free quantization algorithms were published \cite{nagel2019data,meller2019same,zhao2019improving,cai2020zeroq,zhang2021diversifying,squant2022}, focusing on the quantization operator, \textit{i.e.} the transformation which maps the floating point weights to their low-bit, fixed point, values. However, these approaches still struggle to offer an interesting alternative to data-driven techniques in terms of accuracy preservation. Furthermore, when considering a specific target device for deployment, traditional quantization methods, usually focusing on the quantization operator, offer limited options: given a supported bit width (given by the device, as most hardware usually support only a few representation formats \cite{nivdiaA100}) they either achieve satisfactory accuracy or not.

In our previous work, REx \cite{yvinec2023rex}, we introduced a novel residual expansion of quantized deep neural networks. Drawing inspiration from wavelets-based methods for image compression \cite{rabbani2002jpeg2000,mallat2009theory}, this approach considers successive residual quantization errors between the quantized and original model. As such, REx can provide several accuracy \textit{vs.} speed trade-off points for each bit width. Increasing the number of residuals in the expansion (\textit{i.e.} the expansion order) increases the fidelity to the original, non-quantized model at the expanse of additional computations. In addition, we proposed a group-sparse expansion, which allows us to maintain the accuracy using significantly less bit operations.

In this work, we propose an extension of this work, dubbed PIPE, that leverages parallelization capabilities of modern hardware. Formally, from a residual expansion of a quantized model, we can group together several terms in the expansion and, under certain assumptions, completely separate the computation between several subnetworks of a resulting ensemble. This ensemble approximation, depending on the parallelization capacities of the target hardware, may result in dramatic speed enhancement. Therefore, given a target device, our approach allows finding the best accuracy \textit{vs.} speed trade-offs. 
Our contributions are thus three-fold:
\begin{itemize}
    \item \textbf{\ours, a data-free quantization method that is both efficient and flexible.} \ours leverages residual quantization, with an ensemble approximation, to enable finding suitable trade-offs depending on a target bit-width and parallelization capacity.
    \item \textbf{Theoretical guarantees} on both the exponential convergence of the quantized model towards the full-precision model, and ensemble approximation errors. This is of paramount importance in a data-free context, where we cannot easily measure the accuracy degradation.
    \item \textbf{Extensive empirical validation} we show through a thorough validation that \ours significantly outperforms every state-of-the-art data-free quantization technique as a standalone method but also helps improve said methods when used in combination. In particular, \ours achieves outstanding performances on both standard and low bit range quantization on various ConvNet architectures applied to ImageNet classification, Pascal VOC object detection and CityScapes semantic segmentation.
\end{itemize}
In addition, PIPE is agnostic to the quantization operator and can be combined with most recent state-of-the-art methods that focus on the latter.

\section{Related Work}
\subsection{Quantization}
In this section, we review existing methods for DNN quantization, with an emphasis on approaches geared towards run-time acceleration.
The vast majority of DNN quantization techniques rely on data usage (Quantization-Aware Training) and \cite{wu2018training,jacob2018quantization,achterhold2018variational,louizos2018relaxed,sheng2018quantization,ullrich2017soft,zhou2016dorefa} usually rely on variants of the straight through estimation to alleviate the rounding operation gradients.
Among these methods, \cite{oh2021automated} bears the most resemblance with the proposed \ours method. It minimizes the residual error during training, using weight decay over the residue. The similarity with \ours comes from the use of a second order expansion of the quantization errors. However, it discards the quantization error after training, while we propose to keep the extra operations in order to ensure a high fidelity to the provided pre-trained model.

\subsection{Data-Free Quantization}
% Authors in \cite{nagel2019data} discuss the necessity to have data available to successfully design a quantization pipeline. They proposed a method that consists in balancing the weight ranges over the different layers of a model, using scale invariance properties that are specific to piece-wise affine (e.g. ReLU) activation functions, and relying on a traditional, naive quantization operator \cite{krishnamoorthi2018quantizing}. 
% In the current state of data-free quantization research, we see two major trends: methods that focus on the rounding operator itself \cite{squant2022,yvinec2022spiq} and methods that generate synthetic data \cite{li2021mixmix,choi2022s,zhong2022intraq}.
% With \ours, we propose to enable hardware flexibility for these methods by allowing to find better trade-offs in terms of accuracy and compression rate given a fixed bit-width.
Nagel et al. \cite{nagel2019data} discuss the necessity to have data available to successfully design a quantization pipeline. The proposed method consists in balancing the weight ranges over the different layers of a model, using scale invariance properties (similarly to \cite{stock2019equi}) that are specific to piece-wise affine (e.g. ReLU) activation functions, and relying on a traditional, naive quantization operator \cite{krishnamoorthi2018quantizing}. The authors note that the magnitude of the quantization error strongly varies with the DNN architecture: as such, already compact architectures such as MobileNets \cite{sandler2018MobileNetV2} appear as challenging for data-free quantization purposes (for instance, authors in \cite{krishnamoorthi2018quantizing} report a dramatic drop to chance-level accuracy without fine-tuning). Lin \textit{et al.} \cite{lin2016fixed} studied the properties of the noise induced by the quantization operator. These properties were later used in SQNR \cite{meller2019same}, a method that consists in assigning, for each layer, an optimal bit-width representation.
Overall, data-free approaches generally struggle to deal with low-bit representation problem, \textit{i.e.} performing quantization into bit widths lower than int4 (e.g. int3 or ternary quantization). However, the proposed method successfully addresses this challenge with the addition of the residual errors in higher order expansions. Furthermore, to compensate for the increased latency, we propose a budgeted expansion as well as a rewriting of the quantized expansion as ensembles for further parallelization.

\subsection{Flexibility in Quantization}
In practice, the existing data-free quantization methods only offer a single possible quantized model given a supported bit-width. Nevertheless, most pieces of hardware do not support a wide range of bit-width. For instance, Turing \cite{feng2021apnn} and Untether AI \cite{robinson_2022} architectures support int4 and int8 quantization while the Nvidia A100 \cite{nivdiaA100} supports int8, int4 and binary (int1) quantization. Conversely, \ours circumvents this limitation by offering several trade-offs given a bit-width representation. As discussed by \cite{zhang2022pokebnn}, hardware cost is for the most part derived from energy consumption. Consequently, quantizing models to supported bit width is a problem of paramount importance.

\subsection{Ensemble Methods}
Ensemble methods \cite{dietterich2000ensemble} are ubiquitous and widely studied models in the machine learning community, where weak, yet complimentary predictors are aggregated \textit{via} e.g. bagging \cite{breiman1996bagging}, boosting \cite{friedman2000additive} or gradient boosting \cite{breiman1997arcing}, to achieve superior performance. Leveraging deep learning and ensemble methods crossovers is still an overlooked subject, partly because standalone DNNs are usually quite robust on their own, and because DNNs already involve a high computational burden. Nevertheless, some methods \cite{rame2021dice,guo2015deep,tan2019evolving,arnaud2022thin} leveraged deep ensembles to great success for various applications. Of particular interest is the work of Zhu \textit{et al.} \cite{zhu2019binary} which consists in learning ensembles of binary neural networks (BNNs) to reach interesting accuracy vs. inference speed trade-offs, thanks to the potential weakness of predictors working with very low-bit representations such as BNNs, the potential complementarity between these, as well as the fact that, intrinsically, ensembles can be easily parallelized in practice. Our method offers several advantages over \cite{zhu2019binary}: First, \ours is applied to existing DNNs in a data-free manner, thus can be applied to accelerate already performing networks without bells and whistles. More importantly, accuracy of the BNN ensembles is significantly lower than that of the original full-precision model, while \ours achieves high acceleration without significant accuracy loss. In addition, results reported in \cite{zhu2019binary} are admittedly unstable as the ensembles grow, which the authors attribute to overfitting. \ours, however, is robust to such problems, as we demonstrate a convergence to the original accuracy with respect to the order of expansion.

\section{Methodology overview}

Let's consider $F$, a trained network with $L$ layers and trained weights $W_l$. Given a target integer representation in $b$ bits, e.g. int8 or int4, we consider a quantization operator $Q$. Formally, $Q$ maps $[\min\{W_l\}; \max\{W_l\}] \subset \mathbb{R}$ to the quantized interval $ [- 2^{b-1} ; 2^{b-1} -1] \cap \mathbb{Z}$. The most straightforward way to do so is to apply a scaling $s_l$ and round $\lfloor \cdot \rceil$ the scaled tensor, \textit{i.e.}:
\begin{equation}\label{eq:quantization_operator}
    Q(W_l) = \left\lfloor\frac{W_l}{s_{W_l}}\right\rceil
\end{equation}
With $s_l$ the quantization scale for $W_l$ computed as in \cite{krishnamoorthi2018quantizing} without loss of generality. Following the standard formulation \cite{gholami2021survey}, a quantization operator $Q$, comes with a de-quantization operator $Q^{-1}$. For the simple quantization operator $Q$ in Equation \eqref{eq:quantization_operator}, a natural choice is $Q^{-1}(Q(W_l)) = s_l \times Q(W_l)$. Note that, despite the notation, $Q^{-1}$ is not a true inverse, as by definition of the quantized space, there is some information loss. This loss, called the quantization error, is defined as: $W_l - Q^{-1}(Q(W_l))$.
In data-free quantization, we want to minimize this error in order to achieve the highest possible fidelity to the original model.
In the following section, we describe how we can efficiently reduce the quantization error for a fixed target bit-width $b$.

\begin{figure*}[!t]
    \centering
    \includegraphics[width = 0.85\linewidth]{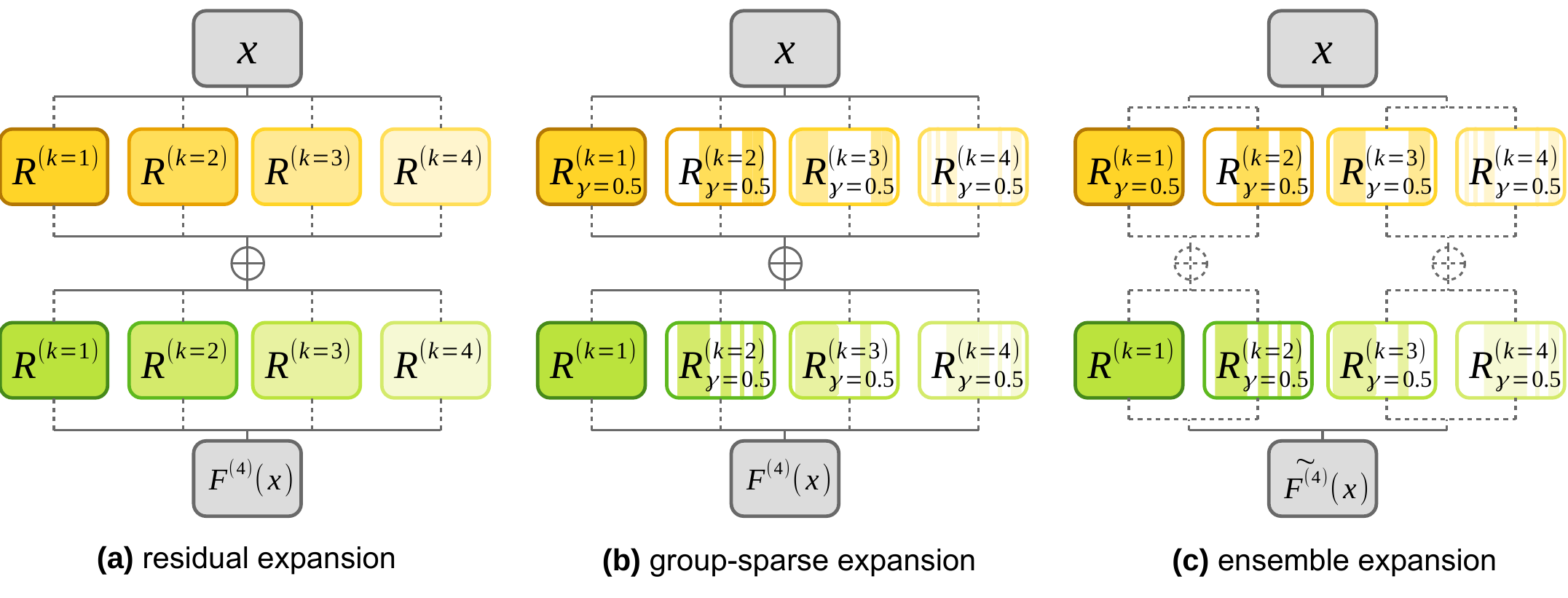}
    \caption{Illustration of the proposed method for a two-layers neural network. \textbf{(a)} residual expansion at order $4$: the intensity of the colormap indicates the magnitude of the residual error. \textbf{(b)} group-sparse expansion for orders $k \geq 1$ ($\gamma = 50\%$ sparsity). \textbf{(c)} ensemble expansion with two predictors. Dashed lines indicate parallel computations.}
    \label{fig:DRE_explain}
\end{figure*}

\subsection{Residual Expansion}
We propose to quantize the residual errors introduced by the quantization process. Although the proposed method can be applied to any tensor, let's consider a weight tensor $W$. In the full-precision space ($\mathbb{R}$), its first approximation is $R^1 = Q^{-1}(Q(W))$. To reduce the quantization error, we define $R^2$ as the quantized residual error
\begin{equation}\label{eq:residu_2_definition}
    R^{2} = Q^{-1}\left(Q\left(W - R^1 \right)\right)
\end{equation}
Consequently, during the quantized inference, we compute $R^1X + R^2X\approx WX$ which provides a finer approximation than the simple evaluation $R^1X$. For the sake of generality, we will not necessarily assume that all weights were expanded, \textit{i.e.} some weights may have been pruned like in \cite{yvinec2023rex}. The process can be generalized to any expansion order $K$. 
\begin{equation}\label{eq:residu_definition}
    R^{K} = Q^{-1}\left(Q\left(W - \sum_{k=1}^{K-1} R^k \right)\right)
\end{equation}
The resulting expanded layer is illustrated in Figure \ref{fig:DRE_explain} (a) in the case $K=4$.
Intuitively, an expansion $(R^1,...,R^K)$ provides the approximation $\sum_{k=1}^K R^k$ of $W$ and this approximation converges exponentially fast to the original full-precision weights with respect to $K$. As the support of the quantization error space is smaller than one quantization step, the error decreases by a factor larger than $2^b$ with each expansion term. Furthermore, as the quantization error decreases, it is expected that the prediction of the quantized model shall achieve a closer match to the original predictions. This is especially important in the context of data-free quantization, as we do not have the option to perform fine-tuning to recover accuracy. Worst, we also can not evaluate the degradation of the model on a calibration/validation set. Nonetheless, in \cite{yvinec2023rex} we provided an upper bound on the maximum error $\epsilon_{\max}$ introduced by residual quantization on the predictions, as 
\begin{equation}\label{eq:main_upper_bound}
    \epsilon_{\max} \leq U = \prod_{l=1}^L \left(\sum_{i=1}^l \left(\frac{1}{2^{b-1}-1}\right)^{K-1}\frac{s_{R^{i}}}{2} + 1 \right) - 1
\end{equation}
where $s_i$ is the scaling factor from equation \ref{eq:quantization_operator} applied to each residue. This implies that, in practice and regardless of the quantization operator, a network can be quantized with high fidelity with only a few expansion orders to fit a given bit-width. 

In its most general expression, the residual expansion can be defined with a pruning mask $M^k$ for each residual. The pruning mask can be applied in either a structured or unstructured fashion, depending on the desired outcome. For instance, we showed that it is very effective to handle LLMs outliers \cite{yvinec2023rex}. 
\begin{equation}
    R^{K} = M^K Q^{-1}\left(Q\left(W - \sum_{k=1}^{K-1} R^k \right)\right)
\end{equation}
In the case of unstructured and structured pruning, $M^K$ is a sparse tensor or tensor was row-wise zeros, respectively. In the remainder of this study, we will assume a structured mask (simpler to leverage) unless stated otherwise. In practice, the pruning ratio is given by a sparsity parameter $\gamma$ in order to match a total budget of operations with $K$. This sparse expansion is guaranteed to converge faster than the standard one to the original weight values. 

In the resulting expanded quantized network, the inference of the residual weights $R^{k}$ can be performed in parallel for each layer. In \ours, we enable to perform the inference of residual weights across different layers in parallel through ensembling.

\subsection{Ensembling from Expansion}

\begin{figure}[!t]
    \centering
    \includegraphics[width = 0.7\linewidth]{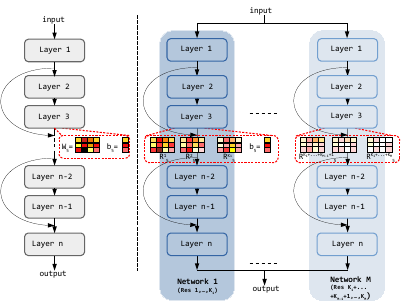}
    \caption{Example of ensemble expansion with groupings [$K_1$, $\dots$, $K_M$] for a residual network. The original network is broken down into $M$ networks, each having exactly the same architecture as the original network, except that (a) the weights for a layer $l$ of network $m \in \{1,\dots, M\}$ correspond to the residual expansion terms $R^{K_1+\dots+K_{m-1} + 1}_l \dots R^{K_1+\dots+K_m}_l$, and (b) the bias terms for all layers appear only in Network 1. At inference time, the input is fed to all the $M$ networks, resulting in efficient parallelization at virtually no cost in term of accuracy, depending on the expansion term clustering.}
    \label{fig:ensapprox}
\end{figure}

So far (see Figure \ref{fig:DRE_explain} (a-b)) each layer computes the $R^k \times X$ for all orders $k$. As previously stated, this formulation allows finding better trade-offs depending on hardware capacities. However, it does not fully exploit the potential parallelization capacities of the hardware, as the results from all these expansion orders have to be summed before applying the non-linearity. Intuitively, a better way to leverage parallelization would be to only sum the results after the last layer (Figure \ref{fig:DRE_explain} (c)), akin to an ensemble model where the elements in the ensemble corresponds to (a clustering of) the different expansion orders $k$.

To do so, we exploit two assumptions. First, the activation functions $\sigma$ of the model satisfy the following:
\begin{equation}\label{ensemble_approx}
\sigma(\cdot+\epsilon) \approx \sigma(\cdot) + \sigma(\epsilon)
\end{equation}
When $|\epsilon| \rightarrow 0$. This holds true for most popular activation functions such as ReLU, SiLU, GeLU or sigmoid. Second, the exponential convergence of the expansion terms (Equation \ref{eq:main_upper_bound}) ensures that the first expansion orders are preponderant w.r.t. the subsequent ones. With this in mind, we can group the expansion orders in $M$ clusters that each contain $K_m$ successive expansion orders, with $m \in [|1,M|]$ and [$K_1$, $\dots$, $K_M$] ($K_1+\dots+K_M = K$) the total number of orders in the expansion. For each of these clusters, the sum of the expansion orders that are contained inside must have negligible dynamics with regard to the previous cluster (see \ref{sec:appendix_ensembling} on how to empirically group expansion orders) to successively apply the approximation in Equation \ref{ensemble_approx}. Finally, we can define the $M$ quantized networks of the ensemble as having the same architecture as the original model $F$, except that the biases of the model are all assigned to $F_1$. For each $m = 1,\dots,M$, the weights of the $m^{\text{th}}$ predictor corresponds to the sum over the residuals at orders belonging to the  $m^{\text{th}}$ cluster. This ensemble approximation is illustrated on Figure \ref{fig:ensapprox}.

Proof of this approximation can be found in \ref{appendix:ensembling}. This ensemble approximation (Figure \ref{fig:DRE_explain} (c)) also comes with strong theoretical guarantees (see \ref{appendix:upperbound}) on accuracy preservation depending on expansion order grouping. Furthermore, it allows better usage of the potential parallelization capacities of a target hardware (per-model instead of per-layer parallelization), as will be shown in the upcoming experiments.

\section{Quantization Experiments}

In the following sections, we first go through the implementation requirements and efficient strategies to fully leverage the proposed expansions. Second, we perform a comparison of each expansion method in order to show the flexibility of \ours with respect to the bit-width. Third, we compare \ours to other quantization schemes under the constraint of equal bit operations as well as under the assumption of heavy parallelization for ensembling. Finally, we validate for each expansion their respective upper bound on the maximum error with respect to the original predictions.

\subsection{Implementation Details and Benchmarks}

We ran our tests on 6 different backbones, including ConvNets and transformers, and 4 tasks from both computer vision and natural language processing. We used ImageNet \cite{imagenet_cvpr09}, Pascal VOC 2012 \cite{pascal-voc-2012}, CityScapes dataset \cite{cordts2016cityscapes} and GLUE \cite{wang-etal-2018-glue}.

Unless stated otherwise, we apply symmetric, static, per-channel quantization as defined in \cite{gholami2021survey} and perform batch-normalization folding prior to any processing using the optimal method by \cite{yvinec2022fold}. In order to leverage the existing efficient implementations of the convolutional layers and fully-connected layers in CUDA, we propose to implement the expanded layer using a single kernel rather than $K$ kernels. This is achieved by concatenating the kernels along the output dimension. Consequently, the challenge of efficiently splitting the computations to fully leverage the target device computational power is left to the inference engine. In practice, this results in both better performance and less work, in order to adapt the method to existing engines such as OpenVino \cite{openvino} and TensorRT \cite{tensorrt}. Furthermore, the sparse expansion does not use sparse matrix multiplications as sparsity is applied to the neurons (or channels for ConvNets). The libraries, pre-trained model checkpoints and datasets information, are detailed in \ref{appendix:implem}. We evaluate the pre-processing time required by \ours and compare it to other generic data-free quantization methods in \ref{appendix:preprocessingtime}. In the following section, we confirm the hinted benefits from each of the proposed expansions.

\subsection{Flexible Quantization}
\begin{figure}[!t]
    \centering
    \includegraphics[width = 0.7\linewidth]{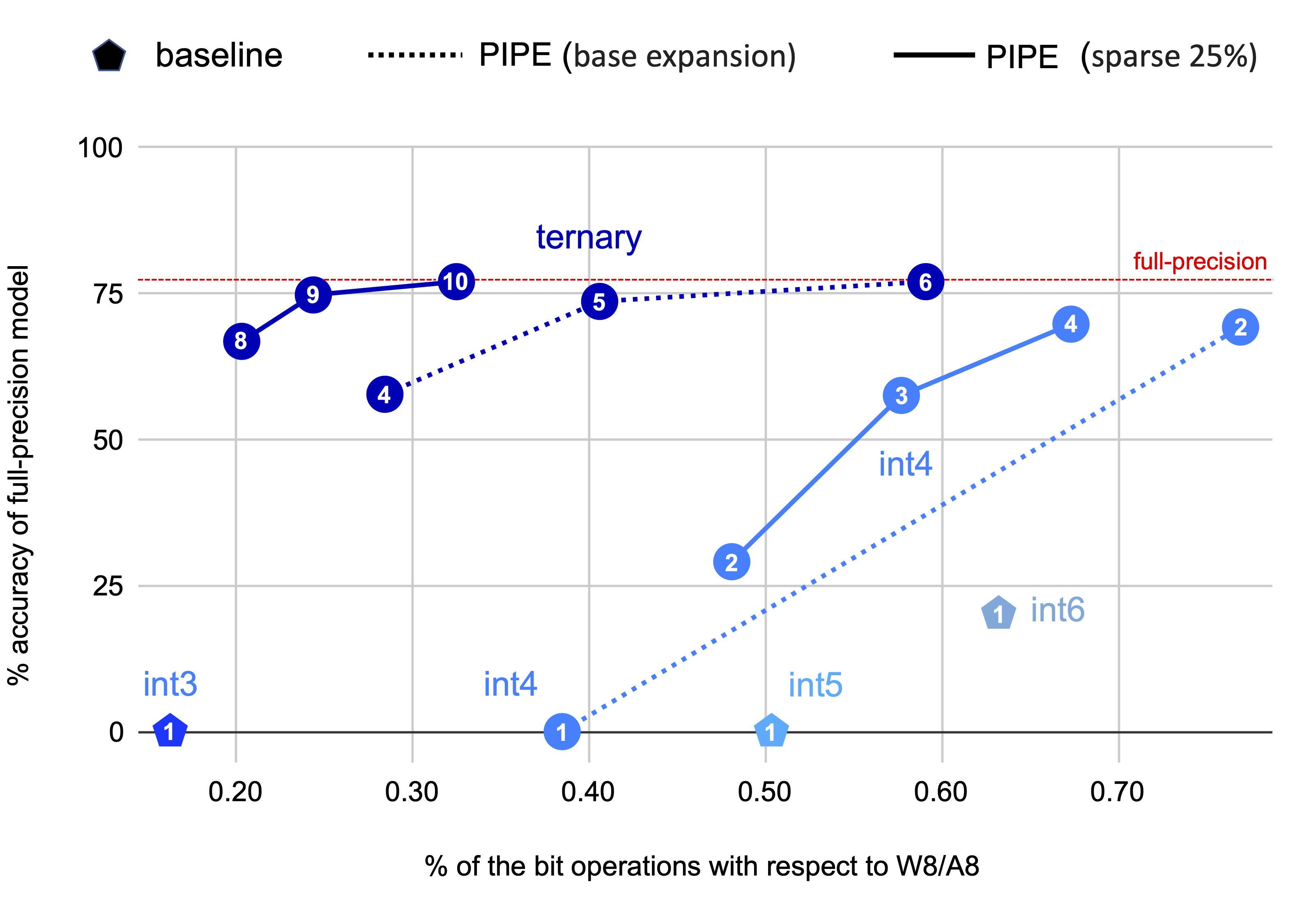}
    \caption{Accuracy \textit{vs.} inference time, for EfficientNet B0. The higher (accuracy) and the further to the left (inference cost) the better. The pentagons show the baseline results with W3/A3, W4/A4, W5/A5 and W6/A6 quantization. The dashed lines show the trade-offs in performance of \ours in W4/A4 and ternary quantization. Finally, the plain lines show \ours (with sparsity) also in W4/A4 and ternary quantization. The numbers in the symbols stands for the expansion order. \ours, and \textit{a fortiori} the sparse version, enables better trade-offs.}
    \label{fig:ablation}
\end{figure}

Figure \ref{fig:ablation} shows different trade-offs enabled by \ours on different bit-widths for an EfficientNet-B0 on ImageNet. First, the baseline quantization with the baseline quantization operator from \cite{krishnamoorthi2018quantizing} (as depicted by the pentagon of different colors-one for each bit width) offers no trade-off possibility given a specific bit-width and usually performs poorly below int8 quantization (e.g. barely reaching $20.290\%$ top1 accuracy in W6/A6 quantization). \ours, however, in the same setup, enables finding several trade-offs for each specific bit-width (e.g. int4 and ternary on Figure \ref{fig:ablation}) and supporting hardware. Furthermore, the sparse expansion enables finding more potential trade-offs (by varying the budget and expansion order) for every bit-width. Those trade-offs are generally more interesting than comparable ones obtained using the base expansion.

We do not require the use of extremely sparse residues in order to get the best accuracy. For instance, in Figure \ref{fig:ablation} we reach full-precision accuracy using 25\% sparse residues. In other words, the process converges fast with respect to the sparsity rates. All in all, these results demonstrate the flexibility of \ours to find good accuracy \textit{v.s.} speed trade-offs, given a budget of total bit operations (BOPs) to fit. In the following section, we evaluate the ability of \ours to outperform existing quantization methods in terms of equal bops as well as in the context of heavy parallelization.

\subsection{Main Results}\label{compsota}
\subsubsection{Equal BOPs}
\begin{table}[!t]
\caption{Comparison at equal BOPs (\textit{i.e.} no-parallelization) with existing methods in W6/A6 and \ours with W4/A6 +50\% of one 4 bits residue. In all tested configurations, distributing the computations between the residuals in a lower bit format enables to find superior trade-offs.}
\label{tab:compar_sota_equalbits}
\centering
\setlength\tabcolsep{3pt}
        % \vspace*{-0.5\baselineskip}
        \begin{tabular}{c|c|c|c|c}
         \hline
         DNN & method & year & bits & Accuracy \\
         \hline
         \multirow{8}{*}{ResNet 50} & \multicolumn{3}{c|}{full-precision} & 76.15 \\
         \cline{2-5}
         & \hyperlink{cite.nagel2019data}{DFQ} & ICCV '19 & W6/A6 & 71.36 \\
         & \hyperlink{cite.cai2020zeroq}{ZeroQ} & CVPR '20 & W6/A6 & 72.93\\
         & \hyperlink{cite.zhang2021diversifying}{DSG} & CVPR '21 & W6/A6 & 74.07 \\
         & \hyperlink{cite.xu2020generative}{GDFQ} & ECCV '20 & W6/A6 & 74.59 \\
         & \hyperlink{cite.squant2022}{SQuant} & ICLR '22 & W6/A6 & 75.95 \\
         & \hyperlink{cite.yvinec2022spiq}{SPIQ} & WACV '23 & W6/A6 & 75.98 \\
         & \ours & - & 150\% $\times$ W4/A6 & \textbf{76.01} \\
         \hline
         \multirow{5}{*}{MobNet v2} & \multicolumn{3}{c|}{full-precision} & 71.80 \\
         \cline{2-5}
        & \hyperlink{cite.nagel2019data}{DFQ} & ICCV '19 & W6/A6 & 45.84 \\
        & \hyperlink{cite.squant2022}{SQuant} & ICLR '22 & W6/A6 & 61.87 \\
        & \hyperlink{cite.yvinec2022spiq}{SPIQ} & WACV '23 & W6/A6 & 63.24 \\
        & \ours & - & 150\% $\times$ W4/A6 & \textbf{64.20} \\
         \hline
         \multirow{4}{*}{EffNet B0} & \multicolumn{3}{c|}{full-precision} & 77.10 \\
         \cline{2-5}
        & \hyperlink{cite.nagel2019data}{DFQ} & ICCV '19 & W6/A6 & 43.08 \\
        & \hyperlink{cite.squant2022}{SQuant} & ICLR '22 & W6/A6 & 54.51 \\
        & \ours & - & 150\% $\times$ W4/A6 & \textbf{57.63} \\
         \hline
        \end{tabular}
\end{table}
In order to highlight the benefits of residual quantization errors expansions as a stand-alone improvement upon existing methods with equal BOPs, we compare \ours using the naive quantization operator from \cite{krishnamoorthi2018quantizing} on a variety of reference benchmarks. First, in Table \ref{tab:compar_sota_equalbits}, we report the performance on three different computer vision networks between state-of-the-art methods in W6/A6 quantization (other setups are discussed in \ref{appendix:more_results}) and \ours using a sparse expansion at order $K=2$ using $50\%$ of a 4 bits representations in order to get a similar total number of bit operations (150\% of 4 bits $\approx$ 6 bits). For all networks, \ours significantly outperform recent state-of-the-art data-free quantization methods at equal BOPs. Furthermore, we confirm these results on object detection and image segmentation in \ref{appendix:more_results}.

\begin{table*}[!t]
\caption{GLUE task quantized in W4/A8. We consider the BERT transformer architecture \cite{devlin2018bert} and provide the original performance (from the article) of BERT on GLUE as well as our reproduced results (reproduced). \ours is applied to the weights with 3 bits + 33\% sparse expansion.}
\label{tab:comparison_sota_bert}
\centering
\setlength\tabcolsep{4pt}
\begin{subtable}{.44\textwidth}
% \centering
\raggedleft
% \hspace{3cm}
\begin{tabular}{|c|c|c|}
\hline
task & original & reproduced \\
\hline
CoLA & 49.23 & 47.90 \\
SST-2 & 91.97 & 92.32 \\
MRPC & 89.47 & 89.32 \\
STS-B & 83.95 & 84.01 \\
QQP & 88.40 & 90.77 \\
MNLI & 80.61 & 80.54 \\
QNLI & 87.46 & 91.47 \\
RTE & 61.73 & 61.82 \\
WNLI & 45.07 & 43.76 \\
\hline
\end{tabular}
\end{subtable}
\begin{subtable}{.55\textwidth}
% \hfill
\raggedright
% \centering
\begin{tabular}{|c|c|c|c|c|c|}
\hline
\hyperlink{cite.krishnamoorthi2018quantizing}{uniform} & \hyperlink{cite.zhou2017incremental}{log} & \hyperlink{cite.squant2022}{SQuant} & \hyperlink{cite.yvinec2022spiq}{SPIQ} & \ours \\
\hline
45.60 & 45.67 & \underline{46.88} & 46.23 & \textbf{47.02} \\
\underline{91.81} & 91.53 & 91.09 & 91.01 & \textbf{91.88} \\
88.24 & 86.54 & \textbf{88.78} & 88.78 &  \underline{88.71}\\
\underline{83.89} & {84.01} & 83.80 & 83.49 & \textbf{83.92} \\
89.56 & 90.30 & \underline{90.34} & 90.30 & \textbf{90.50} \\
\underline{78.96} & 78.96 & 78.35 & 78.52 & \textbf{79.03} \\
89.36 & 89.52 & \textbf{90.08} & \underline{89.64} & \textbf{90.08} \\
\underline{60.96} & 60.46 & 60.21 & 60.21 & \textbf{61.20} \\
39.06 & 42.19 & \underline{42.56} & 42.12 & \textbf{42.63} \\
\hline
\end{tabular}

\end{subtable}
\end{table*}
Second, In Table \ref{tab:comparison_sota_bert}, we perform a similar experiment on NLP using Bert \cite{devlin2018bert}. Similarly to our results on ConvNets, \ours can find better accuracy per bit trade-offs as compared to the fours references, including non-uniform quantization \cite{miyashita2016convolutional}. Furthermore, if we consider parallelization, \ours can offer even higher accuracy results using ensemble approximation, as shown in what follows.

\subsubsection{Leveraging Parallelization}

\begin{figure}[!t]
    \centering
    \includegraphics[width = 0.8\linewidth]{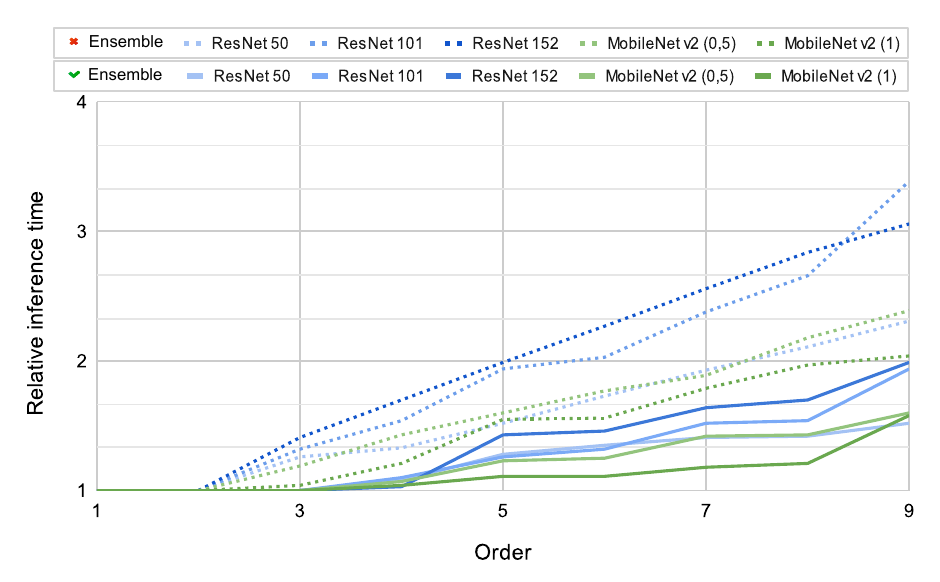}
    \caption{Standardized inference time on ImageNet of different architectures. We demonstrate that parallelization of the overhead computations brought by the proposed ensemble approximation drastically reduces their impact on runtime on an intel m3 CPU.}
    \label{fig:inference}
\end{figure}

On large devices using CPUs or GPUs for inference, parallelization of the computations within a layer or a model can drastically reduce the runtime. In Figure \ref{fig:inference}, we showcase the normalized inference times (\textit{i.e.} the ratio between the runtime of the expanded networks and the baseline quantized model) for several ResNets and MobileNets on ImageNet.
On the one hand, as indicated by the dashed plots (no ensembling), the relative inference time grows sub-linearly  for each network, e.g. order 2 comes at virtually no cost in terms of inference time, while using higher orders may induce a noticeable slow-down: $<50\%$ speed reduction for order $K=3$, and about $100\%$ at order $K=5$. On the other hand, when we evaluate ensembles (plain lines), and specifically two predictors with similar sizes (see \ref{sec:appendix_ensembling}), we observe that we can expand the quantized models up to order $4$ without noticeably degrading the inference speed even a small CPU such as the intel m3. This is due to the more efficient parallelization using the proposed ensemble approximation.

Consequently, in Table \ref{tab:compar_sota} we compare \ours to other data-free quantization methods under this assumption of heavy parallelization (and asymmetric quantization). We consider an ensemble of two weak predictors with each 2 expansion orders: $R^1,R^2$ for the first predictor and $R^3,R^4$ for the second. Our results on the challenging MobileNet show that, as compared to the most recent data-free quantization operators that do not use data-generation (no-DG) such as SQuant \cite{squant2022} and SPIQ \cite{yvinec2022spiq}, the proposed \ours method improves the accuracy by 16.35 points in 4 bits quantization. Furthermore, data-free techniques that leverage synthetic data are usually known for their higher performance as compared to the methods that only focus on the quantization operator. Nonetheless, \ours still manages to improve upon recent techniques including IntraQ \cite{zhong2022intraq} and AIT \cite{choi2022s} by 5.26 points in the advantageous and realistic assumption of heavy parallelized inference. These observations can be generalized to more DNN architectures, as discussed in \ref{appendix:parallelization_extra}.
We therefore conclude that \ours allows to find better accuracy \textit{v.s.} speed trade-offs given the possibility to leverage parallelization on larger hardware.

\begin{table}[!t]
\caption{Accuracy for MobileNet V2 on ImageNet with 8 bits activations. \ours uses ensembling of two weak predictors each sharing the same number of bit operations and similar runtime, based on Figure \ref{fig:inference}.}
\label{tab:compar_sota}
\centering
\setlength\tabcolsep{4pt}

      \centering
        \begin{tabular}{c|c|c|c|c}
        \hline
            method & year & no-DG & bits & accuracy\\
        \hline
        \multicolumn{4}{c|}{full-precision} & 71.80\\
        \hline 
        \hyperlink{cite.zhao2019improving}{OCS} & ICML 2019 & \cmark & W8/A8 & 71.10 \\
        \hyperlink{cite.nagel2019data}{DFQ}	& ICCV 2019 & \cmark & W8/A8 & 70.92 \\
        \hyperlink{cite.meller2019same}{SQNR} & ICML 2019 & \cmark & W8/A8 & 71.2 \\
        \hyperlink{cite.cai2020zeroq}{ZeroQ} & CVPR 2020 & \xmark & W8/A8 & 71.61 \\
        \hyperlink{cite.yvinec2022spiq}{SPIQ} & WACV 2023 & \cmark & W8/A8 &  71.79\\
        \hyperlink{cite.xu2020generative}{GDFQ} & ECCV 2020 & \xmark & W8/A8 & \textbf{71.80} \\
        \ours & - & \cmark & W8/A8 & \textbf{71.80} \\
        \hline
        \hyperlink{cite.nagel2019data}{DFQ} & ICCV 2019 & \cmark & W4/A8 & 0.10\\
        \hyperlink{cite.cai2020zeroq}{ZeroQ} & CVPR 2020 & \xmark & W4/A8 & 49.83\\ % source MixMix
        \hyperlink{cite.xu2020generative}{GDFQ} & ECCV 2020 & \xmark & W4/A8 & 51.30 \\
        \hyperlink{cite.squant2022}{SQuant} & ICLR 2022 & \cmark & W4/A8 & 55.38\\
        \hyperlink{cite.li2021mixmix}{MixMix} & CVPR 2021 & \xmark & W4/A8 & 65.38\\
        \hyperlink{cite.choi2022s}{AIT} & CVPR 2022 & \xmark & W4/A8 & 66.47 \\
        \hyperlink{cite.zhong2022intraq}{IntraQ} & CVPR 2022 & \xmark & W4/A8 & 65.10\\
        \ours & - & \cmark & W4/A8 & \textbf{71.73} \\
        \hline
        \end{tabular}
\end{table}

\subsection{Empirical Validation of the Theoretical Bounds}

\begin{table}[!t]
\caption{Upper bound $U$ (see theorem \ref{thm:dre_upperbound}, \ref{thm:slim_upperbound} and \ref{thm:ensemble}) over the maximum error as compared to the corresponding empirical measurement $U_{\text{empirical}}$ of that error for a VGG 16 \cite{simonyan2014very} trained on ImageNet. The closer the upper bound $U$ to the value $U_{\text{empirical}}$ the better.}
\label{tab:upper_bound}
\centering
\setlength\tabcolsep{2pt}
\begin{tabular}{c|c|c|c|c|c}
\hline
bits & K & sparsity & ensemble & $U$ & $U_{\text{empirical}}$\\
\hline
8 & 1 & \xmark & \xmark & 0.12 & 0.05\\
8 & 4 & \xmark & \xmark & 1.99 $\times 10^{-7}$ & 1.78 $\times 10^{-7}$ \\
8 & 2 & 50\% & \xmark & 0.06 & 0.05 \\
8 & 4 & 50\% & \xmark & 1.17 $\times 10^{-7}$ & 0.65 $\times 10^{-7}$ \\
8 & 2 & \xmark & \cmark & 0.09 & 0.02 \\
8 & 4 & \xmark & \cmark & 0.47 $\times 10^{-4}$ & 0.43 $\times 10^{-4}$\\
\hline
\end{tabular}
\end{table}

In Table \ref{tab:upper_bound}, we validate the proposed upper bound $U$ on the maximum error on the predictions (see Equation \ref{eq:main_upper_bound}) on a VGG-16 \cite{simonyan2014very} trained on ImageNet. The tightness of the provided theoretical results can be estimated from the gap between our estimation and the empirical maximum error $U_{\text{empirical}}$ from quantization on the predictions, which is measured as the infinite norm between the full-precision and quantized logits. We observe that a naïve 8-bits quantization (\textit{i.e.} no expansion) leads to an upper bound $U = 0.12$, while we observe $U_{\text{empirical}}=0.05$. Compare with the norms of the logits, which in this case is equal to $0.3423$: as such, the proposed upper bound appears as relatively tight and significantly lower than the logits magnitude. In such a case, due to overconfidence, the error shall not, in theory, affect the classification. The proposed upper bound is even tighter for larger values of $K$, and becomes lower and lower (for both the theoretical and corresponding empirical maximum errors) when introducing sparsity. Last but not least, we see on the last two rows in Table \ref{tab:upper_bound} that $U$ stays tight when using the ensemble approximation. This further demonstrates the good properties of the proposed expansion, sparse expansion and ensemble approximation in \ours in addition to the relevance of its theoretical guarantees, which are critical in data-free quantization.

\subsection{Flexibility with respect to the Quantization Operator}
\begin{table}[!t]
    \centering
    \setlength\tabcolsep{2pt}
    \caption{We report the different trade-offs achieved with \ours expanding over different proposed quantization operators in W4/A4 as compared to their performance in W8/A8, on a MobileNet V2.}
    \begin{tabular}{c|c|c|c|c|c|c}
\hline
    method & W4 & $\text{W4}_{\text{+ 25\%}}$ & $\text{W4}_{\text{+ 50\%}}$ & $\text{W4}_{\text{+ 75\%}}$ & W6 &  W8 \\
\hline
\hline
        \hyperlink{cite.krishnamoorthi2018quantizing}{naive} & 0.1 & 53.11 & 64.20 & \textbf{71.61} & 51.47 & 70.92 \\
        \hyperlink{cite.squant2022}{SQuant} & 4.23 & 58.64 & 67.43 & \textbf{71.74} & 60.19 & 71.68 \\
        \hyperlink{cite.yvinec2022spiq}{SPIQ} & 5.81 & 59.37 & 68.82 & \textbf{71.79} & 63.24 & \textbf{71.79} \\
        \hyperlink{cite.nagel2020up}{AdaRound} & 6.17 & 60.30 & 69.80 & \textbf{71.77} & 68.71 & \textbf{71.75} \\
        \hyperlink{cite.li2021brecq}{BrecQ} & 66.57 & 70.94 & 71.28 & \textbf{71.76} & 70.45 & \textbf{71.76} \\
\hline
    \end{tabular}
    \label{tab:flexibility_operator}
\end{table}

Most recent approaches for data-free quantization focus on designing better quantization operators. Interestingly, our approach is agnostic to the choice of the quantization operator and can thus be combined with these approaches without bells and whistles. In Table \ref{tab:flexibility_operator}, we report the possible trade-offs achievable with \ours combined with recent approaches focusing on the quantization operator, on MobileNet V2. The different trade-offs are sorted in ascending order in terms of added overhead operations, e.g. $\text{W4}_{\text{+ 25\%}}$ leads to fewer operations than $\text{W4}_{\text{+ 50\%}}$.
First, when used with SQuant \cite{squant2022}, \ours achieves full-precision accuracy in W4/A4 with only $75\%$ overhead, even outperforming W8/A8 quantization. SPIQ \cite{yvinec2022spiq}, can also be adapted with \ours in order to achieve good accuracy using only 4 bits representation as it benefits from finer weight quantization. This explains the slightly higher accuracies than SQuant using 25\% and 50\% sparsity. Finally, with AdaRound \cite{nagel2020up} and BrecQ \cite{li2021brecq}, two PTQ techniques, we observe similar results as expected.
In particular, BrecQ which already achieves decent accuracy in W4/A4 with a $5.23$ points accuracy drop gets closer to the original accuracy ($0.86$ point accuracy drop) using a quarter of the expansion.
In such a case, if the target hardware only has support for W4 and W6, \ours shall allow using W4 quantization with a small overhead due to the addition of the sparse expansion term (which can be parallelized using the proposed ensemble approximation if the hardware supports it), whereas most methods would stick with W6 quantization to preserve the accuracy of the model. Thus, we believe that those results illustrate the adaptability of the proposed framework. 

\section{Conclusion}

In this work, we proposed a novel ensemble approximation for the residual expansion of the quantization error to design a novel flexible data-free quantization method. In order to find the best accuracy \textit{v.s.} speed trade-offs, we proposed to consider the residuals of the quantization error, along with a group-sparse formulation. Furthermore, we showed that, under certain assumptions, the terms in the residual expansion can be grouped together at the whole network level, resulting in an ensemble of (grouped together) residual expansion networks. This ensemble approximation allows to efficiently leverage the parallelization capacities of the hardware, if any.

The proposed method, dubbed \ours, is flexible with respect to both hardware bit-width support and parallelization capacities. In addition, PIPE is backed up with strong theoretical guarantees which, critical in the context of data-free quantization (where one can not empirically estimate the accuracy degradation), as we provide a tight upper bound on the error caused by the residual quantization and ensemble approximations.

Through extensive experimental validation, we showed the benefits of the proposed approach. As such, PIPE significantly outperforms recent data-free quantization methods on a wide range of ConvNet architectures applied to ImageNet classification, Pascal VOC object detection, CityScapes semantic segmentation as well as transformers on GLUE text classification. Furthermore, as showed in previous work \cite{yvinec2023rex}, the proposed framework also constitutes an elegant way of dealing with outlying values in LLMs: as such, it appears as an ideal choice for designing flexible quantization approaches to further reduce the memory footprint and latency of DNNs. This is of paramount importance since nowadays models become more and more parameter and computation hungry. Last but not least, the ideas presented in this paper are orthogonal to most recent approaches focusing on improving the quantization operator, and hence can straightforwardly be combined with those approaches. 

\subsection{Limitations:} The residual expansion method introduced in this paper does not adapt to the inter-layer importance and runtime cost discrepancies. An interesting future work would thus consist in applying more expansion orders on the most important layers w.r.t. the model accuracy, as well as using fewer orders for the most computationally expensive layers.

\bibliographystyle{elsarticle-num} 
\bibliography{egbib}

\newpage
\appendix
\onecolumn

\tableofcontents

\section{How to set predictors size}\label{sec:appendix_ensembling}

\begin{figure}[!t]
    \centering
    \includegraphics[width = 0.7\linewidth]{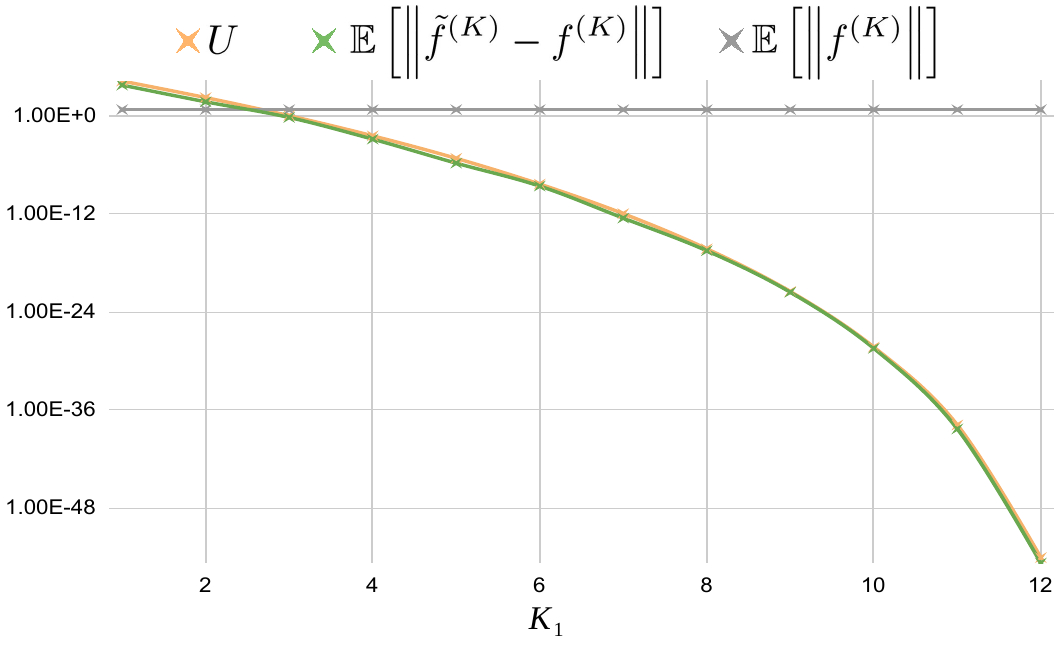}
    \caption{Comparison between the expected empirical error from ensembling $\mathbb{E}[\|f^{(K)}-\tilde f^{(K)}\|]$ (green) and its upper bound $U$ (Lemma \ref{thm:ensemble}, orange) for different values of $K_1$ on a ResNet 50 trained on ImageNet and quantized with ternary values and $K=13$, $\gamma = 25\%$. We also plot the reference $\mathbb{E}[\|f^{(K)}\|]$ (grey).}
    \label{fig:U}
\end{figure}
\begin{figure}[!t]
    \centering
    \includegraphics[width = 0.7\linewidth]{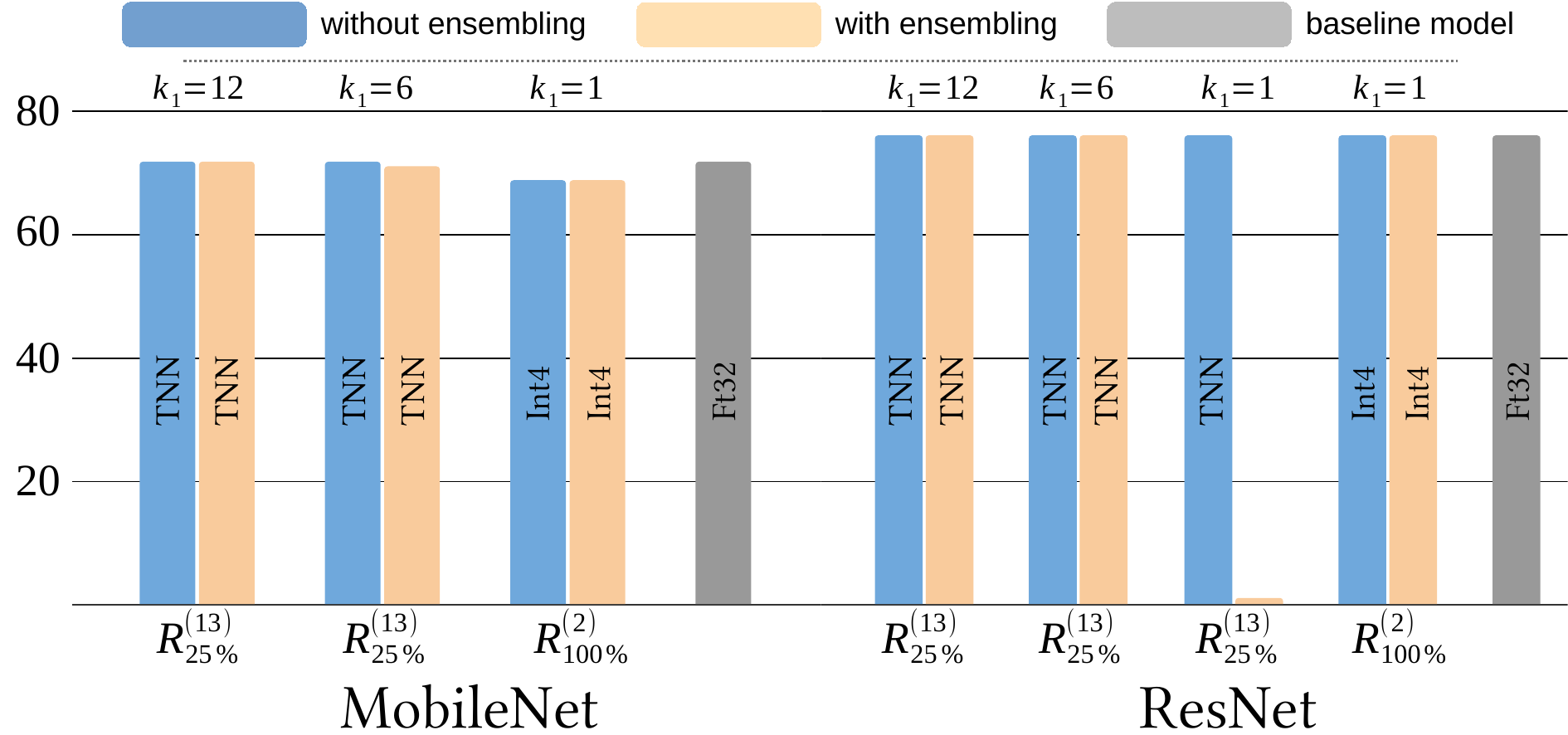}
    \caption{Comparison between ensemble expansion $\tilde f^{(K)}$ (in orange) and regular expansion $f^{(K)}$ (blue) on ImageNet. We test different bit representations, namely ternary (TNN) and int4 as well as different values for $K_1$. Except for very low values of the ratio $K_1/K$, we observe the robustness of the ensembling method.}
    \label{fig:ensembling_2}
\end{figure}
We consider quantized expansions of networks with $M$ predictors such that $\Sigma_{m=1}^M K_m = K$ ($K_1$ is the number of orders in the first predictor of the ensemble) and $\gamma$ the sparsity factor. The larger $K_1$, the lower the difference between the ensemble $\tilde f^{(K)}$ and the developed network $f^{(K)}$. Conversely, the more balanced the elements of the ensemble, the more runtime-efficient the ensemble: thus, $K_1$ have to be fixed carefully so that the ensemble shall be faster than the developed network, without accuracy loss. Fortunately, the accuracy behavior w.r.t. the value of $K_1$ can be estimated from the values of the upper bound $U$ (Lemma \ref{thm:ensemble}) on the expected error from ensembling $\mathbb{E}[\|f^{(K)}-\tilde f^{(K)}\|]$. As illustrated on Figure \ref{fig:U} in the case of ternary quantization, this upper bound is tight and collapses more than exponentially fast w.r.t. $K_1$. For instance, if $K_1 \leq 2$, $U$ is significantly larger than the amplitude of the logits $\mathbb{E}[\|f^{(K)}\|]$ and the accuracy is at risk of collapsing. When $U$ vanishes compared to $\mathbb{E}[\|f^{(K)}\|]$, the ensemble and regular expansions are guaranteed to be almost identical, and the accuracy is preserved. Thus, we can compare the upper bound $U$ and the empirical norm of the logits from the expansion $\mathbb{E}[\|f^{(K)}\|]$ to assess the validity of an ensemble. Plus, $\mathbb{E}[\|f^{(K)}\|]$ can be estimated using statistics from the last batch norm layers to allow for a fully data-free criterion.

\begin{figure}[!t]
    \centering
    \includegraphics[width = 0.6\linewidth]{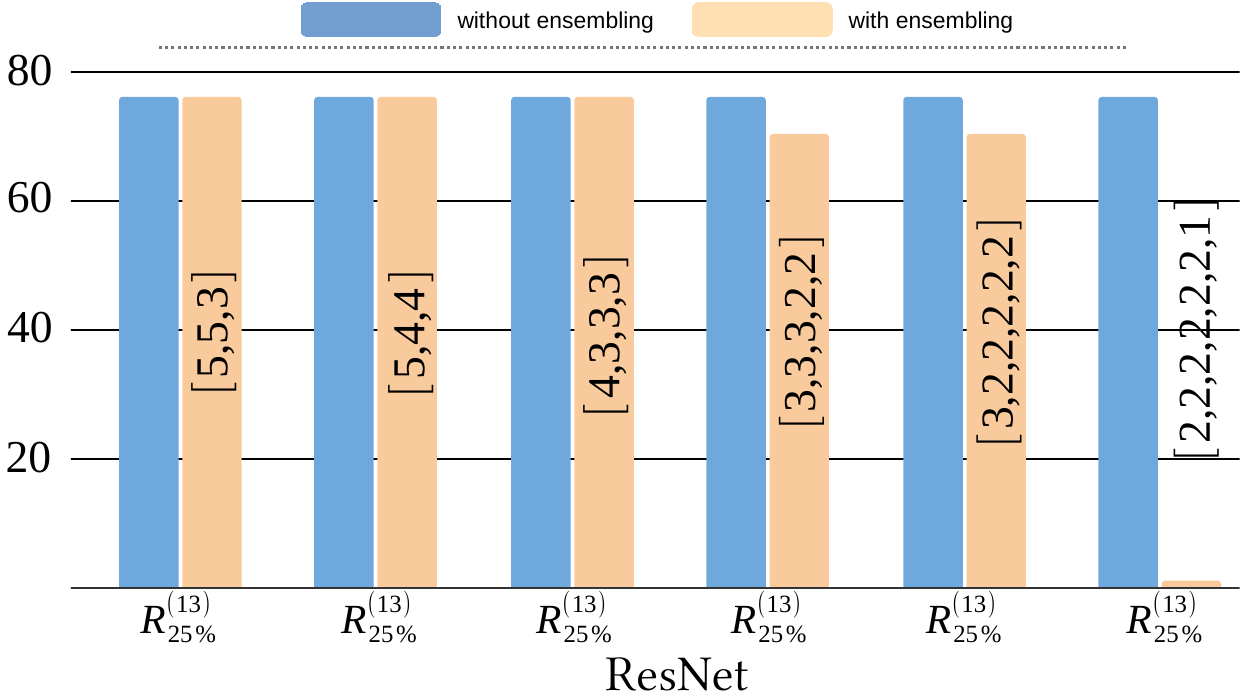}
    \caption{Comparison between TNN ensemble expansion $\tilde f^{(K)}$ (in orange) and regular expansion $f^{(K)}$ (blue) on ImageNet.}
    \label{fig:ensembling_3+}
\end{figure}

With this in mind, in Figure \ref{fig:ensembling_2} we compare the top-1 accuracies of $\tilde f^{(K)}$ and $f^{(K)}$ for different architectures (MobileNet v2 and ResNet 50) and quantization configurations. The ensemble expansion systematically matches the accuracy of the developed network in terms of accuracy, except in the case of ternary quantization, when $K_1 =1$. This is remarkable, as ensembling significantly decreases the inference time with a two predictors configuration.

Figure \ref{fig:ensembling_3+} shows the results obtained with larger ensembles of smaller quantized predictors, \textit{i.e.} with $M>2$.
We observe the full preservation of the accuracy of the developed network as long as $K_1 \geq 4$ and a loss of 6 points for balanced ensembles of $5-6$ predictors and $K_1=3$. Here again, with $M=7$ and $K_1=2$, the accuracy is very low, as predicted by \ref{fig:U}. To sum it up, ensembling developed networks allows to significantly decrease the inference time, with theoretical guarantees on the accuracy preservation.

\begin{table}[!t]
\caption{Comparison of the evaluation time in seconds of a ResNet 50 over the validation set of ImageNet using an expansion of order $k=8$ with ensembling of $m$ predictors $m\in\{1,2,3,4\}$. We distinguish the setups by reporting the values of $[K_1,...,K_m]$.}
\label{tab:run_time}
\centering
\setlength\tabcolsep{4pt}

\begin{tabular}{c|c|c|c|c||c}
\hline
 device & $[8]$ & $[4,4]$ & $[3,3,2]$ & $[2,2,2,2]$ & $[1]$ \\
 expansion & \cmark & \cmark & \cmark & \cmark & \xmark \\
 ensembling & \xmark & \cmark & \cmark & \cmark & \xmark \\
\hline\hline
Intel(R) i9-9900K   & 215 & 54 & 26 & 26 & 25 \\
Ryzen 3960X     & 122 & 30 & 23 & 22 & 22 \\
RTX 2070                     & 13 & 8 & 5 & 5 & 5 \\
RTX 3090                     & 11 & 7 & 4 & 4 & 4 \\
\hline
\end{tabular}
\end{table}

Finally, Table \ref{tab:run_time} shows the runtime of a ResNet 50 for a full evaluation on ImageNet validation set (50,000 images). We tested the models on different devices (CPU/GPU) using a fixed budget $\beta = 7$ and order $K = 8$, and compared ensembles expansions (with 2 $[4,4]$, 3 $[3,3,2]$ and 4 $[2,2,2,2]$ predictors). On each device, the ensembles are up to 10 times faster than the baseline expansion.

\section{Ensembling Protocol}\label{appendix:ensembling}
We recall some results from previous work \cite{yvinec2023rex}:
\begin{lemma}\label{thm:dre}
Let $f$ be a layer with weights $W \in \mathbb{R}^n$ with a symmetric distribution. 
We denote $R^{(k)}$ the $\text{k}^{\text{th}}$ quantized weight from the corresponding residual error.
Then the error between the rescaled $W^{(K)}=Q^{-1}(R^{(K)})$ and original weights $W$ decreases exponentially, \textit{i.e.}:
\begin{equation}\label{eq:exponential_decrease}
    \left| w - \sum_{k = 1}^{K} w^{(k)} \right| \leq \left(\frac{1}{2^{b-1}-1}\right)^{K-1} \frac{{\left(\lambda_{R^{(K)}}\right)}_i}{2}
\end{equation}
where $w$ and $w^{(k)}$ denote the elements of $W$ and $W^{(k)}$ and ${\left(\lambda_{R^{(k)}}\right)}_i$ denotes the row-wise rescaling factor at order $k$ corresponding to $w$, as defined in equation \ref{eq:quantization_operator}.
\end{lemma}
\begin{lemma}\label{thm:slim}
Let $f$ be a layer of real-valued weights $W$ with a symmetric distribution.
Then we have
\begin{equation}\label{eq:slim}
\begin{aligned}
    \left| w - \left(\sum_{k = 1}^{K-1} w^{(k)} + Q^{-1}\left(R^{(K)}_\gamma\right) \right) \right| \\ \leq  \frac{\left\| N^{(K)} \cdot \mathbbm{1}^{(K)}_{\gamma} \right\|_\infty{\left(\lambda_{R^{(k)}}\right)}_i}{\left(2^{b-1}-1\right)^{K}2}
\end{aligned}
\end{equation}
where $\|\|_\infty$ is the infinite norm operator with the convention that $\|0\|_\infty = 1$ and ${\left(\lambda_{R^{(k)}}\right)}_i$ denotes the row-wise rescaling factor at order $K$ corresponding to $w$.
\end{lemma}
Lemma \ref{thm:dre} and \ref{thm:slim} state that the first terms in the expansion, \textit{i.e.} the lower values of $k$, are preponderant within the magnitude before the activation. Moreover, the activation functions traditionally used in DNNs (e.g. ReLU) satisfy $\sigma(x+\epsilon) \approx \sigma(x) + \sigma(\epsilon)$ when $|\epsilon| \rightarrow 0$. With respect to the proposed expansion, $x$ corresponds to the first orders and $\epsilon$ to the terms of higher orders. In the case of two-layers networks, these properties allow us to break down the network in an (approximately) equivalent ensemble of two networks, the first one containing the first, largest orders, and the second one containing the remaining ones.

\subsection{Ensemble of two Layers DNNs}
Let $F$ be a feed-forward DNN with two layers $f_1$,$f_2$ and $\sigma$ a piece-wise affine activation function (e.g. ReLU). Given $(R^{(k)}_1)_{k=1 \dots K}$ and $b_1$ the kernel and bias weights of the first layer $f_1^{(K)}$ respectively, we define the quantization expansion of residual errors $(R^{(k)}_1)_{k\in \{1,...,K\}}$ of order $K$ as in
\begin{equation}\label{eq:define_dre}
    f^{(K)} : x \mapsto \sigma \left( \sum_{k=1}^K R^{(k)}Q(x)\lambda_{R^{(k)}}\lambda_x + b\right).
\end{equation}
Lemma \ref{thm:dre} states that the first terms in the sum, \textit{i.e.} the lower values of $k$, are preponderant in the pre-activation term. Thus, there exists $K_1 < K$ such that $f_1^{(K)} \approx \tilde f_1^{(K)} = \tilde f_{1,1}^{(K)} + \tilde f_{1,2}^{(K)}$ with:
\begin{equation}\label{eq:condition_for_ensembling}
    \begin{cases}
    \tilde f_{1,1}^{(K)} : x \mapsto \sigma\left( \sum_{k=1}^{K_1} R^{(k)}_1x^q \lambda_{R^{(k)}_1} \lambda_x + b_1\right)\\
    \tilde f_{1,2}^{(K)} : x \mapsto \sigma\left(\sum_{k=K_1 + 1}^{K} R^{(k)}_1x^q \lambda_{R^{(k)}_1} \lambda_x\right)
    \end{cases}
\end{equation}
Furthermore $F^{(K)} : x \mapsto f_2^{(K)}(f_1^{(K)}(x))$. 
Let $R^{(k)}_2$ and $b_2$ respectively denote the kernel and bias weights of the second layer $f_2^{(K)}$. By linearity of the last layer, we have
\begin{equation}\label{eq:ensemble_two_layers}
\begin{aligned}
    F^{(K)} \approx \tilde F^{(K)} &= \sum_{k=1}^K R^{(k)}_2 \tilde f_{1,1}^{(K)} \lambda_{R^{(k)}_2} \lambda_{\tilde f_{1,1}^{(K)}} + b_2 \\
    &+ \sum_{k=1}^K R^{(k)}_2 \tilde f_{1,2}^{(K)} \lambda_{R^{(k)}_2} \lambda_{\tilde f_{1,2}^{(K)}}
\end{aligned}
\end{equation}
Stemming from this formulation, we can express the quantized network $f^{(K)}$ as an ensemble of quantized neural networks which share a similar architecture, \textit{i.e.} $F^{(K)} \approx \tilde F^{(K)}=\tilde g^{(K)} + \tilde h^{(K)}$. This defines the ensemble expansion from residual errors of order $K$.

\subsection{Ensemble of more Layers DNNs}\label{sec:appendix_algorithm}

Similarly, we demonstrate by structural induction that a network with arbitrary number of layers $L$ can be approximated by an ensemble expansion $\tilde F^{(K)}$ composed of $M$ quantized networks, defined by the parameters $K_1,...,K_M$ setting the size of each predictor, such that $K_1+\dots+K_M = K$ (the total number of orders in the expansion). 

We recall that $f^{(K)}_{L-1}(x) = \sigma_{L-1}\left(\sum_{k=1}^K R^{(k)}_{L-1}\lambda_{R^{(k)}_{L-1}} X_{f^{(K)}_{L-1}} + b_{L-1}\right)$. If we directly apply equation \ref{eq:condition_for_ensembling} then we get for a given $K_{L-1} < K$
\begin{equation}
\begin{aligned}
    X_{f^{(K)}_{L-1}} \mapsto& \sigma_{L-1}\left( \sum_{k=1}^{K_{L-1}}R^{(k)}_{L-1}X_{f^{(K)}_{L-1}}\lambda_{X_{f^{(K)}_{L-1}}} \lambda_{R^{(k)}_{L-1}} + b_{L-1}\right) + \sum_{k=K_{L-1} + 1}^{K} R^{(k)}_{L-1}X_{f^{(K)}_{L-1}}\lambda_{X_{f^{(K)}_{L-1}}} \lambda_{R^{(k)}_{L-1}}
\end{aligned}
\end{equation}
However the two terms $X_{f^{(K)}_{L-1}(x)}$ inside and outside the activation function are not independent. Furthermore, the terms that compose $X_{f^{(K)}_{L-1}(x)}$, from equation \ref{eq:ensemble_l=3_problem}, do not have the same range values, \textit{i.e.} $\tilde g^{(K)}_{L-2}(X_{\tilde g^{(K)}_{L-2}})\lambda_{\tilde g^{(K)}_{L-2}} >> \tilde h^{(K)}_{L-2}(X_{\tilde h^{(K)}_{L-2}})\lambda_{\tilde h^{(K)}_{L-2}}$. We define the operation $*$ as follows
\begin{equation}
\begin{aligned}
    \tilde f^{(K)}_{L-1}(x) &=  \sigma_{L-1}\Bigg( \sum_{k=1}^{K_{L-1}}R^{(k)}_{L-1}\tilde g^{(K)}_{L-2}(X_{\tilde g^{(K)}_{L-2}})\times\lambda_{R^{(k)}_{L-1}}\lambda_{\tilde g^{(K)}_{L-2}} + b_{L-1}\Bigg)+ \sum_{k=K_{L-1} + 1}^{K} R^{(k)}_{L-1} \tilde g^{(K)}_{L-2}(X_{\tilde g^{(K)}_{L-2}})\lambda_{R^{(k)}_{L-1}}\lambda_{\tilde g^{(K)}_{L-2}} \\
    & + \sigma_{L-1}\Bigg( \sum_{k=1}^{K_{L-1}}R^{(k)}_{L-1}\tilde h^{(K)}_{L-2}(X_{\tilde h^{(K)}_{L-2}})\times\lambda_{R^{(k)}_{L-1}}\lambda_{\tilde h^{(K)}_{L-2}} \Bigg)+ \sum_{k=K_{L-1} + 1}^{K} R^{(k)}_{L-1} \tilde h^{(K)}_{L-2}(X_{\tilde h^{(K)}_{L-2}})\lambda_{R^{(k)}_{L-1}}\lambda_{\tilde h^{(K)}_{L-2}}
\end{aligned}
\end{equation}
Now, we have two independent functions $\tilde g_{L-1}^{(K)}$ and $\tilde h_{L-1}^{(K)}$ such that $\tilde f_{L-1}^{(K)}=\tilde g_{L-1}^{(K)}+\tilde h_{L-1}^{(K)}$, these functions have independent inputs and
\begin{equation}
    \begin{cases}
    \tilde g_{L-1}^{(K)}(X_{\tilde g^{(K)}_{L-2}}) =  \sigma_{L-1}\Bigg( \sum_{k=1}^{K_{L-1}}R^{(k)}_{L-1}\tilde g^{(K)}_{L-2}(X_{\tilde g^{(K)}_{L-2}})\times\lambda_{R^{(k)}_{L-1}}\lambda_{\tilde g^{(K)}_{L-2}} + b_{L-1}\Bigg)\vspace{0.2cm}\\ 
    \qquad\qquad\qquad + \sum_{k=K_{L-1} + 1}^{K} R^{(k)}_{L-1} \tilde g^{(K)}_{L-2}(X_{\tilde g^{(K)}_{L-2}})\lambda_{R^{(k)}_{L-1}}\lambda_{\tilde g^{(K)}_{L-2}}\vspace{0.2cm}\\
    \tilde h_{L-1}^{(K)}(X_{\tilde h^{(K)}_{L-2}}) = \sigma_{L-1}\Bigg( \sum_{k=1}^{K_{L-1}}R^{(k)}_{L-1}\tilde h^{(K)}_{L-2}(X_{\tilde h^{(K)}_{L-2}})\times\lambda_{R^{(k)}_{L-1}}\lambda_{\tilde h^{(K)}_{L-2}}\Bigg)\vspace{0.2cm}\\ 
    \qquad\qquad\qquad + \sum_{k=K_{L-1} + 1}^{K} R^{(k)}_{L-1} \tilde h^{(K)}_{L-2}(X_{\tilde h^{(K)}_{L-2}})\lambda_{R^{(k)}_{L-1}}\lambda_{\tilde h^{(K)}_{L-2}}
    \end{cases}
\end{equation}
This defines an iterative procedure in order to define our ensembling of expansions of residual errors for a feed-forward neural network $f$ with any number $L$ of layers. 

To sum it up, the resulting predictors share an identical architecture up to their respective expansion order, defined by $K_1$.
The crucial difference comes from their weight values, which correspond to different orders of expansion of the full-precision weights.
This is also the case if we want ensembles of three or more predictors.
In such instances, instead of only $K_1$, we would have $K_1,...,K_{m-1}$ for $m$ predictors.

Consequently, every predictor shares the same architecture (without biases) up to their respective expansion order. For each $m = 1,\dots,M$, the $m^{\text{th}}$ predictor corresponds to the residuals at orders $k \in\{\sum_{l=1}^{m-1}K_l + i | i \in \{1,\dots,K_m\}\}$. This ensemble approximation allows to more efficiently parallelize the computations across the expansion orders for improved runtimes. We provide insights on how to set the size of each weak predictor in \ref{sec:appendix_ensembling}

\section{Upper Bound Error}\label{appendix:upperbound}
\begin{theorem}\label{thm:dre_upperbound}
Let $F$ be a trained $L$ layers sequential DNN with ReLU activation $\sigma_{L-1} = \dots = \sigma_1$. We note $s_l$ the largest singular value of $W_l$, \textit{i.e.} the spectral norm of $W_l$. Then we have
\begin{equation}
\begin{aligned}
    \max_{\|X \| = 1} \|F(X) - F(X)^{(K)}\|_\infty \leq U_{\text{res}}\\
    U_{\text{res}} = \prod_{l=1}^L \left(\sum_{i=1}^l s_i u_i^{(K)} + 1 \right) - 1
\end{aligned}
\end{equation}
where $u_l^{(K)} = \left(\frac{1}{2^{b-1}-1}\right)^{K-1} \frac{{\left(\lambda_{R^{(K)}}\right)}_i}{2}$ from equation \ref{eq:exponential_decrease}.
\end{theorem}
\begin{proof}
Let's consider $L=2$, and $F : X \mapsto B \sigma (Ax)$. For any $X$ in the domain of $F$ such that $\| X \| = 1$, we have
\begin{equation}
    \| F(X) \|_2 \leq s_B + s_A + s_B s_A
\end{equation}
where $s_B$ is the largest singular value of $B$ and $s_A$ is the largest singular value of $A$. Following the definition of the $2$-norm and $\infty$-norm, we get that 
\begin{equation}
    s_{A - A^{(K)}} \leq s_A u_A^{(K)}
\end{equation}
where $s_{A - A^{(K)}}$ is the largest singular value of the residual error of order $K$, $A - A^{(K)}$ and $u_A^{(K)}$ is derived from equation \ref{eq:exponential_decrease}. Consequently, we get 
\begin{equation}
    \| F(X) - F^{(K)}(X) \|_2 \leq s_B u_B^{(K)} + s_A u_A^{(K)} + s_B u_B^{(K)} s_A u_A^{(K)}
\end{equation}
\end{proof}

\paragraph{Sparse Expansion}
\begin{theorem}\label{thm:slim_upperbound}
Let $F$ be a trained $L$ layers sequential DNN with ReLU activation $\sigma_{L-1} = \dots = \sigma_1$. We note $s_l$ the largest singular value of $W_l$, \textit{i.e.} the spectral norm of $W_l$. Then we have
\begin{equation}
\begin{aligned}
    \max_{\|X \| = 1} \|F(X) - F(X)^{(K)}\|_\infty \leq U_{\text{sparse}}\\
    U_{\text{sparse}}= \prod_{l=1}^L \left(\sum_{i=1}^l s_i u_i^{(K)} + 1 \right) - 1
\end{aligned}
\end{equation}
where $u_l^{(K)} = \frac{\left\| N^{(K)} \cdot \mathbbm{1}^{(K)}_{\gamma} \right\|_\infty{\left(\lambda_{R^{(k)}}\right)}_i}{\left(2^{b-1}-1\right)^{K}2}$ from equation \ref{eq:slim}.
\end{theorem}
This result is directly derived from Theorem \ref{thm:dre_upperbound}.

\paragraph{Ensemble Expansion}

with the following theorem:

\begin{theorem}\label{thm:ensemble}
Let $F$ be a $L$ layers feed-forward DNN with ReLU activation. The expected error $ \mathbb{E}\left[\left\|F^{(K)} - \tilde F^{(K)}\right\|\right]$ due to the ensemble expansion at order $K$ with $M$ predictors is bounded by $U_{\text{ens}}$ which can be approximated as:
\begin{equation}
    U_{\text{ens}} \approx \frac{\sum_{k=1}^{K_1} \|R_{L}^{(k)}\|}{2^{L-1}} \left(1 + \sum_{m=2}^{M} \prod_{l=1}^{L-1} \sum_{k=K_m}^{K_{m+1}} \|R_{L}^{(k)}\|\lambda_{\tilde h_l^{(k)}} \right)
\end{equation}
\end{theorem}

The upper bound $U_{\text{ens}}$ is directly deduced from the largest eigenvalues and reduces as the size of the expansion diminishes. Moreover, the larger $K_1$ the faster the convergence, which is intuitive as in our approximation $\sigma(x+\epsilon) \approx \sigma(x) + \sigma(\epsilon)$ relies on the fact that $x >> \epsilon$. Thus, the ensembling approximation is a way to find new trade-offs in terms of accuracy and parallelization of the inference.
To sum it up, any deep neural network can be approximated by an ensemble expansion of quantized networks, with theoretical guarantees of the approximation error.
In practice, as we will showcase in the experiments, this ensemble approximation from expansion of residual errors leads to superior performances in terms of accuracy-inference time trade-off.

We provide the following intermediate result regarding two-layers DNNs.
\begin{lemma}\label{thm:ensemble_two_layers}
Let $f$ be two layers feed-forward DNN with activation function $\sigma = \text{ReLU}$. The expected error $\mathbb{E}\left[ \left\|f^{(K)} - \tilde f^{(K)}\right\| \right]$ due to the ensemble expansion of order $K$ is bounded by $U$ defined as:
\begin{equation}
    \begin{aligned}
    U =& \sum_{k=1}^K \left(1-\mathbb{P}\left(f_1^{(K)} > 0 \cup \tilde f_{1,1}^{(K)} > 0\right)\right)\lambda_{\tilde f_{1,2}^{(K)}} \lambda_{R^{(k)}_{2}}\|R^{(k)}_2\|\\
    &\times\mathbb{E}\left[\left\| \tilde f_{1,2}^{(K)} \right\|\right]
    \end{aligned}
\end{equation}
where $\|W \|$, for any set of weights $W$, denotes the norm operator or equivalently the spectral norm.
\end{lemma}
\begin{proof}
By definition of the ReLU activation function, if we have $f_1^{(K)} > 0 $ then the activation function of $f_1$ behaves as the identity function.
Similarly, if $\tilde f_{1,1}^{(K)} > 0$ then the activation function of $\tilde f_1$ also behaves as the identity. 
Therefore, if we have $(f_1^{(K)} > 0) \cup (\tilde f_{1,1}^{(K)} > 0)$, then $\tilde f_1^{(K)} = f_1^{(K)}$. We deduce that $\mathbb{E}\left[ \left\|f^{(K)} - \tilde f^{(K)}\right\| \right]$ is equal to
\begin{equation}
    \int_{\{f_1^{(K)} > 0 \cup \tilde f_{1,1}^{(K)} > 0\}^C} \left\|f^{(K)}(x) - \tilde f^{(K)}(x)\right\| \mathbb{P}dx
\end{equation}
where $A^C$ the complementary set of a set $A$ and $x$ is the input. 
In the set defined by $\tilde f_{1,1}^{(K)}(x) = 0$, the value of $\tilde f_1^{(K)}(x)$ is the value of $\tilde f_{1,2}^{(K)}(x)$. If we also have $f_1^{(K)}(x) = 0$ then $\|f^{(K)}(x) - \tilde f^{(K)}(x)\| = \|\tilde f_{1,2}^{(K)}(x)\|$. We can deduce
\begin{equation}
\begin{aligned}
    \mathbb{E}\left[ \left\|f_1^{(K)} - \tilde f_1^{(K)}\right\| \right] &= \left(1-\mathbb{P}\left(f_1^{(K)} > 0 \cup \tilde f_{1,1}^{(K)} > 0\right)\right)\\&\times\mathbb{E}\left[\left\| \tilde f_{1,2}^{(K)} \right\|\right] 
\end{aligned}
\end{equation}
The final result comes from the definition of the norm operator of a matrix and equation \ref{eq:ensemble_two_layers}.
\end{proof}
An immediate limit to lemma \ref{thm:ensemble_two_layers} is the difficulty to compute $1-\mathbb{P}\left(f_1^{(K)} > 0 \cup \tilde f_{1,1}^{(K)} > 0\right)$. However, this value can be approached under the assumption that the distribution of the activations is symmetrical around $0$. Such instances appear with batch normalization layers and result in $1-\mathbb{P}\left(f_1^{(K)} > 0 \cup \tilde f_{1,1}^{(K)} > 0\right) \approx \frac{1}{2}$. 
We also propose to compute the operator norm instead of $\mathbb{E}\left[\left\| \tilde f_{1,2}^{(K)} \right\|\right] $ in order to remain data-free. In consequence, we derive the following corollary.
\begin{corollary}\label{thm:ensemble_two_layers_practical}
The previous upper bound $U$ on the expected error due to the ensemble expansion can be approximated as follows
\begin{equation}
    U \approx \frac{1}{2} \sum_{k=1}^K \|R^{(k)}_2\| \lambda_{\tilde f_{1,2}^{(K)}} \lambda_{R^{(k)}_{2}} \sum_{k=K_1}^K \|R^{(k)}_1\|\lambda_x \lambda_{R^{(k)}_{1}}
\end{equation}
\end{corollary}
In practice, for expansion in $b$ bits, with high values of $b$ (e.g. $b\geq 4$), the single operator $R^{(1)}$ is enough to satisfy equation \ref{eq:condition_for_ensembling} and $K_1 = 1$. For lower values of $b$ (e.g. ternary quantization), the suitable value for $K_1$ depends on the DNN architecture, usually ranging from $3$ to $6$.

\paragraph{Ensembling with more Layers}
We generalize this result to any feed-forward neural network $f$ with $L>2$ layers and activation functions $\sigma_1,...,\sigma_{L-1}$. Equation \ref{eq:define_dre} becomes:
\begin{equation}
    f_L^{(K)}(X_{f_{L-1}^{(K)}}) = \sum_{k=1}^K R^{(k)}_L X_{f_{L-1}^{(K)}} \lambda_{R^{(k)}_L} \lambda_{X_{f_{L-1}^{(K)}}} + b_L
\end{equation}
where $X_{f_{L-1}^{(K)}} = f_{L-1}^{(K)} \circ \dots \circ f_1^{(K)} (x)$.
We reason by induction on the layer $L-1$. Similarly to what precedes, we assume that we have $\tilde g_{L-2}^{(K)},...,\tilde g_1^{(K)}$ and $\tilde h_{L-2}^{(K)},...,\tilde h_1^{(K)}$ such that:
\begin{equation}
    f_{L-2}^{(K)} \circ \dots \circ f_1^{(K)}(x) \approx \tilde g_{L-2}^{(K)}\circ \dots \circ\tilde g_1^{(K)}(x) + \tilde h_{L-2}^{(K)}\circ \dots \circ\tilde h_1^{(K)}(x)
\end{equation}
In order to simplify the composition notations, we note $X_{\mathfrak{f}}$ the input of a function $\mathfrak{f}$. With this in mind Eq. \ref{eq:ensemble_two_layers} becomes:
\begin{equation}\label{eq:ensemble_l=3_problem}
    \begin{cases}
    \tilde f^{(K)}(x) \mkern-18mu &\approx \sum_{k=1}^K R^{(k)}_L\tilde f^{(K)}_{L-1}(x) \lambda_{R^{(k)}_L}\lambda_{\tilde f^{(K)}_{L-1}} + b_L
    \vspace{0.2cm} \\
    \tilde f^{(K)}_{L-1}(x) \mkern-18mu &= \tilde g^{(K)}_{L-1}(X_{\tilde g^{(K)}_{L-1}}) +  \tilde h^{(K)}_{L-1}(X_{\tilde h^{(K)}_{L-1}}) 
    \vspace{0.2cm} \\
    \tilde g^{(K)}_{L-1}(X_{\tilde g^{(K)}_{L-1}})  \mkern-18mu &= \sigma_{L-1}\Bigg( \sum_{k=1}^{K_1}R^{(k)}_{L-1} \tilde g^{(K)}_{L-2}(X_{\tilde g^{(K)}_{L-2}})\times\lambda_{R^{(k)}_{L-1}}\lambda_{\tilde g^{(K)}_{L-2}} + b_{L-1}\Bigg)
    \vspace{0.2cm} \\
    \tilde h^{(K)}_{L-1}(X_{\tilde h^{(K)}_{L-1}}) \mkern-18mu &= \sigma_{L-1}\Bigg( \sum_{k=K_1+1}^{K} R^{(k)}_{L-1} \tilde h^{(K)}_{L-2}(X_{\tilde h^{(K)}_{L-2}})\times\lambda_{R^{(k)}_{L-1}}\lambda_{\tilde h^{(K)}_{L-2}} \Bigg) 
    \vspace{0.2cm} \\
    X_{\tilde g^{(K)}_{L-1}} \mkern-18mu &= \tilde g_{L-2}^{(K)}\circ \dots \circ\tilde g_1^{(K)}(x)
    \vspace{0.2cm} \\
    X_{\tilde h^{(K)}_{L-1}} \mkern-18mu &= \tilde h_{L-2}^{(K)}\circ \dots \circ\tilde h_1^{(K)}(x)
    \end{cases}
\end{equation}
The key element is the definition of $\tilde g^{(K)}_{L-1}$ and $\tilde h^{(K)}_{L-1}$, which are obtained by applying equation \ref{eq:ensemble_two_layers} two times, on $g^{(K)}_{L-2}$ and $h^{(K)}_{L-2}$ independently.
This is detailed in Appendix \ref{appendix:ensembling}.

We showed that we can express $f^{(K)}$ as a sum of two quantized DNNs $\tilde g$ and $\tilde h$.
The first predictor $\tilde g$ is equal to the expansion $f^{(K_1)}$ of $f$ at order $K_1$ while $\tilde h$ is equal to the expansion of $f^{(K)} - f^{(K_1)}$ at order $K - K_1$. This result can be extended to rewrite $f$ as an ensemble of $M$ predictors, by selecting $K_1,...,K_M$ such that $\sum_{m=1}^M K_m = K$: in such a case, the $M^{\text{th}}$ predictor will be the expansion of $f^{(K)}-f^{(\sum_{m=1}^{M-1}K_m)}$ at order $K_M$.

The proof of theorem \ref{thm:ensemble} follows from lemma \ref{thm:ensemble_two_layers} and corollary \ref{thm:ensemble_two_layers_practical}.
We derive the exact formula for the upper bound $U$ in the general case of $L$ layers feed forward neural networks
\begin{equation}
    \sum_{k=1}^K \|R_{L}^{(k)}\|\prod_{l=1}^{L-1}\sum_{k=K_1}^K \mathbb{E}\left[ \left\| \tilde h_{l}^{(k)} \right\| \right]\lambda_{\tilde h^{(k)}_l} P
\end{equation}
where $P=\left( 1 - \mathbb{P}\left(f_l^{(K)} > 0 \cup \tilde f_l^{(K)} > 0\right)\right)$.
This is a consequence of the definition of the operator norm and the proposed ensembling.
The approximation is obtained under the same assumptions and with the same arguments as provided in Corollary \ref{thm:ensemble_two_layers_practical}.

\section{Implementation Details and Datasets}\label{appendix:implem}
We validate the proposed method on three challenging computer vision tasks which are commonly used for comparison of quantization methods.
First, we evaluate on ImageNet \cite{imagenet_cvpr09} ($\approx 1.2$M images train/50k test) classification.
Second, we report results on object detection on Pascal VOC 2012 \cite{pascal-voc-2012} ($\approx$ 17k images in the test set). Third, we benchmark on image segmentation on CityScapes dataset \cite{cordts2016cityscapes} (500 validation images).

In our experiments, we used MobileNets \cite{sandler2018MobileNetV2} and ResNets \cite{he2016deep} on ImageNet. For Pascal VOC object detection we employed an SSD \cite{liu2016ssd} architecture with MobileNet backbone.
On CityScapes we used DeepLab V3+ \cite{chen2018encoder} with MobileNet backbone.
We also test our method on VGG 16 \cite{simonyan2014very} and transformers such as BERT model \cite{devlin2018bert} on GLUE \cite{wang-etal-2018-glue}

In our experiments, the inputs and activations are quantized using the same method as \cite{nagel2019data}. The number of bit-wise operation in our evaluation metric is discussed in Appendix \ref{sec:appendix_head_count}.

For SQuant \cite{squant2022}, we use our own implementation, which achieve different accuracy results due to different initial accuracies for baseline models. As for ZeroQ \cite{cai2020zeroq}, we use results provided by SQuant \cite{squant2022}. Similarly to prior work \cite{meller2019same,nagel2019data,squant2022}, we denote W$\cdot$/A$\cdot$ the quantization setup (number of bits for weight quantization and number of bit for activation quantization).

We used Tensorflow implementations of the baseline models from the official repository when possible or other publicly available resources when necessary.
MobileNets and ResNets for ImageNet come from tensorflow models  \href{https://github.com/tensorflow/models/blob/master/research/object_detection/g3doc/tf2_classification_zoo.md}{zoo}.
In object detection, we tested the SSD model with a MobileNet backbone from \href{https://github.com/ManishSoni1908/Mobilenet-ssd-keras}{Manish's git repository}.
Finally, in image semantic segmentation, the DeepLab V3+ model came from \href{https://github.com/bonlime/keras-deeplab-v3-plus}{Bonlime's git repository}.

The networks pre-trained weights provide standard baseline accuracies on each tasks.
The computations of the residues as well as the work performed on the weights were done using the Numpy python's library.
As listed in Table \ref{tab:exec_time}, the creation of the quantized model takes less than a second for a MobileNet V2 as well as for a ResNet 152 without any optimization of the quantization process.
These results were obtained using an Intel(R) Core(TM) i9-9900K CPU @ 3.60GHz.

\begin{table}[!t]
\caption{Processing time on an Intel(R) Core(TM) i9-9900K CPU @ 3.60GHz of the proposed method for different configurations and architectures trained on ImageNet and a quantization in TNN. We note 'm' for minutes and 's' for seconds.}
\label{tab:exec_time}
\centering
\setlength\tabcolsep{4pt}
\begin{tabular}{c|c|c|c}
\hline
k & ensembling & ResNet 152 & MobileNet v2 (1.4) \\
\hline
1 & \xmark & 0.32s & 0.12s \\
2 & \xmark & 0.43s & 0.13s \\
2 & \cmark & 0.43s & 0.13s \\
7 & \xmark & 0.90s & 0.51s \\
7 & \cmark & 0.92s & 0.51s \\
\hline
\end{tabular}
\end{table}

\section{Pre-Processing Time}\label{appendix:preprocessingtime}
\begin{table}[!t]
\caption{\ours, MixMix, GDFQ and ZeroQ 4-bit quantization time in seconds. \ours was quantized such that full-precision accuracy is reached.}
\label{tab:process_time}
\centering
\setlength\tabcolsep{6pt}
\begin{tabular}{c|c|c|c|c}
\hline
Model & ZeroQ & GDFQ & MixMix & \ours\\
\hline
ResNet 50 & 92.1 & $11.10^{3}$ & $18.10^{3}$ & $<$\textbf{1}\\
ResNet 101 & 164.0 & $18.10^{3}$ & $25.10^{3}$ & $<$\textbf{1}\\
ResNet 152 & 246.4 & $24.10^{3}$ & $30.10^{3}$ & \textbf{1.1}\\
MobileNet V2 (0.35) & 27.4 & $3.10^{3}$ & $6.10^{3}$ & $<$\textbf{1}\\
MobileNet V2 (1) & 37.9 & $7.10^{3}$ & $12.10^{3}$ & $<$\textbf{1}\\
\hline
\end{tabular}
\end{table}

As shown on Table \ref{tab:process_time}, in terms of quantization processing time, methods relying on data generation (DG) such as ZeroQ \cite{cai2020zeroq}, DSG \cite{zhang2021diversifying}, GDFQ \cite{xu2020generative} and MixMix \cite{li2021mixmix} are slow as they usually require many forward-backward passes to quantize a trained neural network. \ours, on the other hand, is very fast in addition to being better at preserving the original model accuracy with theoretical control over the error introduced by quantization.

\section{Other Results on ConvNets}\label{appendix:more_results}
\subsection{More Results on ImageNet}

\begin{table}[!t]
\caption{Comparison at equivalent bit-width (\textit{i.e.} no-parallelization) with existing methods in W8/A8 and \ours with W2/A2 with $K=4$.}
\label{tab:compar_sota_equalbits_ternary}
\centering
\setlength\tabcolsep{1pt}
        % \vspace*{-0.5\baselineskip}
    \begin{subtable}{.5\linewidth}
      \centering
        \begin{tabular}{c|c|c|c|c}
         \hline
         DNN & method & year & bits & Acc \\
         \hline
         \multirow{8}{*}{ResNet 50} & \multicolumn{3}{c|}{full-precision} & 76.15 \\
         \cline{2-5}
         & DFQ & ICCV 2019 & 8 & 75.45 \\
         & ZeroQ & CVPR 2020 & 8 & 75.89\\
         & DSG & CVPR 2021 & 8 & 75.87 \\
         & GDFQ & ECCV 2020 & 8 & 75.71 \\
         & SQuant & ICLR 2022 & 8 & 76.04 \\
         & SPIQ & WACV 2023 & 8 & \textbf{76.15} \\
         & \ours & - & 400\% $\times$2 & \textbf{76.15} \\
         \hline
        \end{tabular}
    \end{subtable}%
    \begin{subtable}{.5\linewidth}
        \begin{tabular}{c|c|c|c|c}
         \hline
         DNN & method & year & bits & Acc \\
         \hline
         \multirow{4}{*}{MobNet v2} & \multicolumn{3}{c|}{full-precision} & 71.80 \\
         \cline{2-5}
        & DFQ & ICCV 2019 & 8 & 70.92 \\
        & SQuant & ICLR 2022 & 8 & \textbf{71.68} \\
        & \ours &  - & 400\% $\times$2  & 71.65 \\
         \hline
         \multirow{4}{*}{EffNet B0} & \multicolumn{3}{c|}{full-precision} & 77.10 \\
         \cline{2-5}
        & DFQ & ICCV 2019 & 8 & 76.89 \\
        & SQuant & ICLR 2022 & 8 & 76.93 \\
        & \ours & - & 400\% $\times$2 & \textbf{76.95} \\
        \hline
        \end{tabular}
    \end{subtable}%
\end{table}

In Table \ref{tab:compar_sota_equalbits_ternary}, we provide complementary results to the comparison at equivalent bit-width using ternary quantization. We see that the results are stable across a wide range of architecture.

In the following sections, we show that \ours is also very flexible and can be straightforwardly applied to other computer vision tasks, e.g. object detection and semantic segmentation.

\subsection{Object Detection}
In Fig \ref{fig:ssd}, we report the performance of \ours (as well as DFQ \cite{nagel2019data} for int8 quantization) using SSD-MobileNet as a base architecture for object detection. Overall, we observe a similar trend as in Section \ref{compsota}: \ours reaches significantly lower numbers of bit-wise operations than the naive baseline ($R^{k=1}_{\gamma=100\%}$) and state-of-the-art DFQ \cite{nagel2019data} while preserving the full model accuracy, using either int4, int3 or ternary quantization. Also, once again, the best results are obtained using ternary quantization with high orders (e.g. $k=8$) and sparse residuals (e.g. $\gamma=25\%$): as such, the mAP of the best tested configuration, $R^{(8)}_{25\%}$, reaches 68.6\% for $6.38e^9$ bit-wise operations, vs. 67.9\% for $3.36e^{10}$ bit-wise operations for DFQ \cite{nagel2019data}.

\subsection{Semantic Segmentation}
In Fig \ref{fig:deeplab}, we report the performance of \ours for image segmentation using a DeepLab v3+ architecture. Similarly to the previous tasks, \ours is able to very efficiently quantize a semantic segmentation network, whether it is in int4 or higher (where order 2 is sufficient to reach the full precision mIoU), or in int3/ternary quantization. In the latter case, once again, it is better to use sparse, high order expansions: for instance, we were able to retrieve the full accuracy using $R^{(9)}_{50\%}$ and ternary quantization, dividing by 10 the number of bit-wise operations as compared to the original model. This demonstrates the robustness of \ours to the task and architecture.

\begin{figure}[!t]
\centering
\includegraphics[width = 0.49\linewidth]{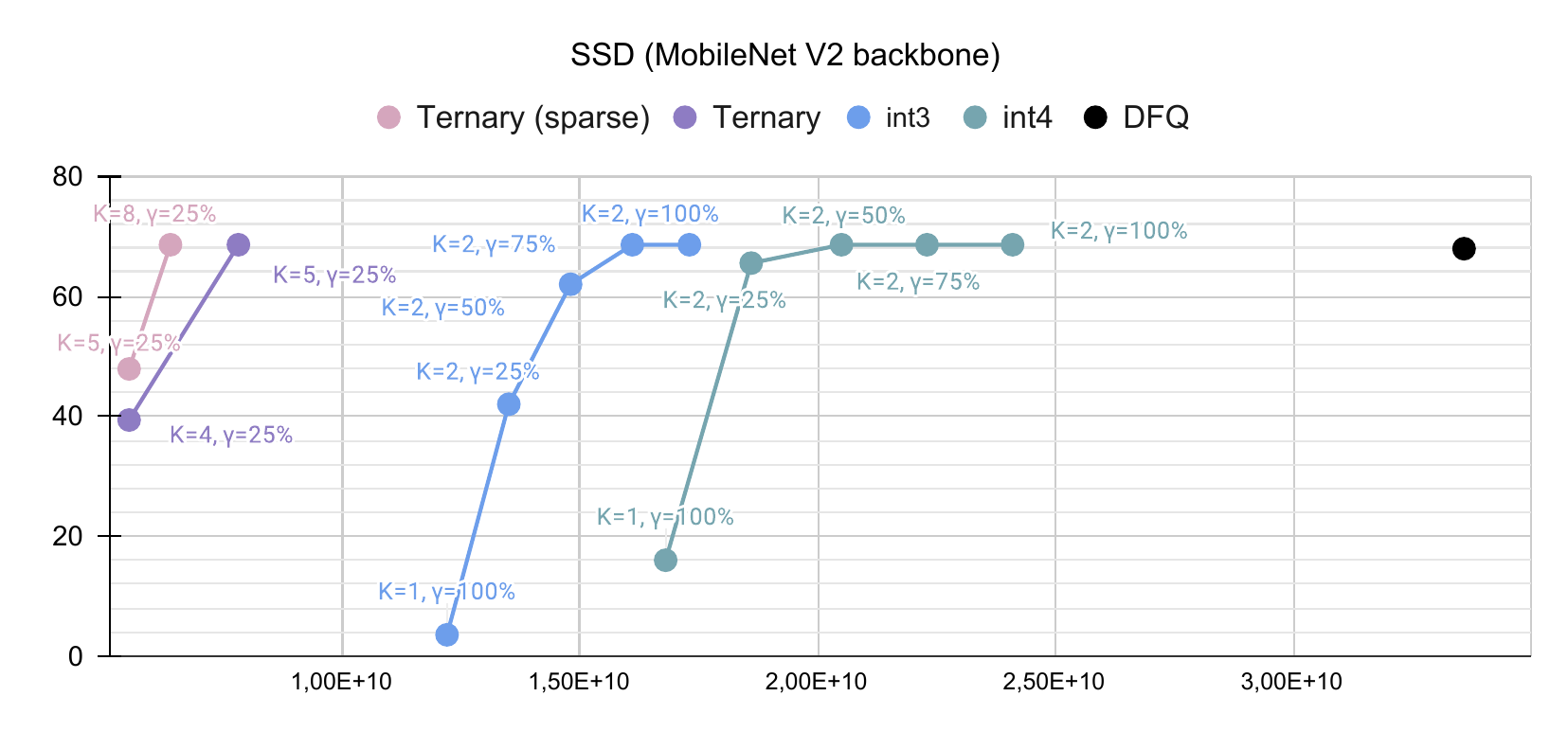}
\includegraphics[width = 0.49\linewidth]{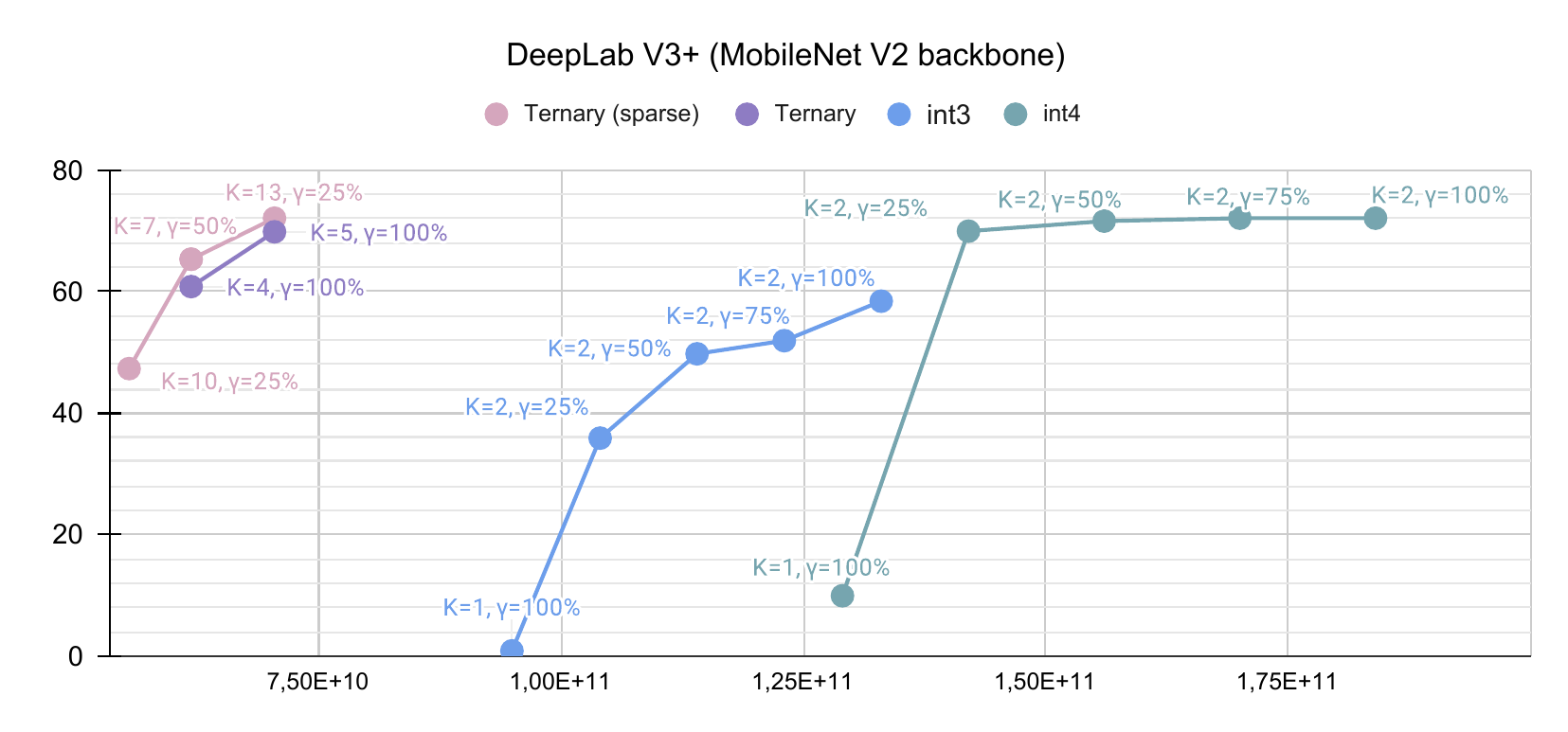}
\caption[Two numerical solutions]{(a) Mean average precision (mAP) of a SSD with MobileNet V2 backbone on Pascal VOC for object detection. We add the performance of a data-free quantization solution, DFQ \cite{nagel2019data} for comparison. (b) Mean intersection over union (mIoU) of a Deeplab V3+ with MobileNet V2 backbone on CityScapes for semantic segmentation.}
\label{fig:ssd}
\label{fig:deeplab}
\end{figure}

\section{More Parallelization Results}\label{appendix:parallelization_extra}

\begin{table}[!t]
\caption{Accuracy for ResNet 50 on ImageNet. \ours (with ensembling)  achieves excellent accuracy for both standard and low-bit quantization, more-so using high-order sparse expansions, vastly outperforming previous state-of-the-art data-free quantization approaches such as OCS, DFQ, SQNR and SQuant, and even approaches that require fine-tuning as MixMix, ZeroQ, DSG and GDFQ.}
\label{tab:compar_sota_bis}
\centering
\setlength\tabcolsep{6pt}
    \begin{subtable}{.49\linewidth}
      \centering
        \begin{tabular}{c|c|c|c|c}
        \hline
        method & year & no-DG & $b$ & accuracy\\
        \hline
        \multicolumn{4}{c|}{full-precision} & 76.15\\
        \hline
        OCS & ICML 2019 & \cmark & 8 & 75.70 \\
        DFQ	& iCCV 2019 & \cmark & 8 & 76.00 \\
        SQNR & ICML 2019 & \cmark & 8 & 75.90 \\
        ZeroQ & CVPR 2020 & \xmark & 8 & 75.89 \\
        DSG & CVPR 2021 & \xmark & 8 & 75.87 \\
        SQuant & ICLR 2022 & \cmark & 8 & 76.04\\ %source spiq
        \ours & - & \cmark & 8 & \textbf{76.15} \\
        \hline
        \end{tabular}
        \end{subtable}
    \begin{subtable}{.49\linewidth}
      \centering
        \begin{tabular}{c|c|c|c|c}
        \hline
        method & year & no-DG & $b$ & accuracy\\
        \hline
        \multicolumn{4}{c|}{full-precision} & 76.15\\
        \hline
        OCS & ICML 2019 & \cmark & 4 & 0.1\\
        DSG &  CVPR 2021 & \xmark & 4 & 23.10\\ % source SQuant
        GDFQ & ECCV 2020 & \xmark & 4 & 55.65\\ % source SQuant
        SQuant & ICLR 2022 & \cmark & 4 & 68.60\\
        MixMix & CVPR 2021 & \xmark & 4 & 74.58\\
        AIT & CVPR 2022 & \xmark & 4 & 66.47 \\
        \ours & - & \cmark & 4 & \textbf{76.13} \\
        \hline
        \end{tabular}
        \end{subtable}
\end{table}

In Table \ref{tab:compar_sota_bis}, we compare \ours and the state-of-the-art data-free quantization methods OCS \cite{zhao2019improving}, DFQ \cite{nagel2019data}, SQNR \cite{meller2019same} and methods that involve synthetic data, such as ZeroQ \cite{cai2020zeroq}, DSG \cite{zhang2021diversifying}, GDFQ \cite{xu2020generative} and MixMix \cite{li2021mixmix}. We report results on int8 and int4 quantization: for the former case, we use order 2 residual expansion without ensembling (as the inference time are equivalent with the baseline model, see Fig \ref{fig:inference}). For int4 quantization, we report results at order $4$ with an ensemble of 2 predictors, each containing two orders, \textit{i.e.} $K_1=K_2=2$. Using this setup ensures that the inference run-times are comparable (see Fig \ref{fig:inference}).

First, given a budget of bit-wise operations that achieves equivalent expected run-time (Fig \ref{fig:inference}), \ours can achieve higher accuracy than existing approaches: e.g. on both MobileNet V2 and ResNet, in int4, \ours outperforms SQuant \cite{squant2022}, the best data-free method that does not involve data-generation (DG), by $16.3$ points and $7.5$ respectively. Note that other methods such as OCS (also involving structural changes in the neural network architecture) considerably underperforms, especially on int4 where the accuracy drops to $0.1$. \ours, however, fully preserves the floating point accuracy.

Second, \ours even outperforms the impressive (yet time consuming) data-generation (DG) based methods such as ZeroQ \cite{cai2020zeroq}, DSG \cite{zhang2021diversifying}, GDFQ \cite{xu2020generative} and MixMix \cite{li2021mixmix}. The difference is more noticeable on low bit quantization, e.g. $b=4$. Nevertheless, \ours improves the accuracy on this benchmark by $6.35\%$ top-1 accuracy. Similarly, to the results on MobileNet V2, on ResNet-50, \ours reaches accuracies very close to the full precision model ($76.13$), significantly outperforming its closest contender, MixMix ($74.58$).

\section{Operation Head-count}\label{sec:appendix_head_count}
Let $W$ be the real-valued weights of a $d\times d$ convolutional layer on input feature maps of shape $D\times D\times n_i$ and $n_o$ outputs and stride $s$.
Then the convolutional product requires $d^2\frac{D^2}{s^2} n_i n_o$ floating point multiplications.
The quantized layer requires two rescaling operations (for the quantization of the inputs and the $Q^{-1}$ operation) and an int-$b$ convolution, i.e. $n_i D^2 + \frac{D^2}{s^2} n_o$ floating point multiplications and $d^2\frac{D^2}{s^2} n_i n_o$ int-$b$ multiplications.
Note that the number of additions remains unchanged.
According to \cite{klarreich2019multiplication} the lowest complexity for $b$-digits scalar multiplication is $o(b\log(b))$ bit operations.
This is theoretically achieved using Harvey-Hoeven algorithm (also the asymptomatic bound has yet to be proved).
We use this value as it is the least favorable setup for the proposed method.
As a consequence the number $O_{\text{original}}$ bit operations required for the original layer, $O_{R^{(1)}}$ the number of bit operations for the naively quantized layer and $O_{R^{(k)}}$ for the $\text{i}^{\text{th}}$ order residual quantization expansion are
\begin{equation}
    \begin{cases}
    O_{\text{original}} = D^2\frac{d^2n_i n_o}{s^2}  32\log(32) \\
    O_{R^{(1)}} = D^2 \left[(n_i  + \frac{n_o}{s^2}) 32\log(32) + \frac{d^2n_i n_o}{s^2} b \log(b)\right] \\
    O_{R^{(k-1)}} = D^2 \left[(n_i  + \frac{n_o}{s^2}) 32\log(32) + k \frac{d^2n_i n_o}{s^2} b \log(b)\right]
    \end{cases}
\end{equation}
Using this result we can estimate the maximum order of expansion before which the number of operations in $f^{(k)}$ exceeds the $O_{\text{baseline}}$.
Note that in the case of fully-connected layers, $D = 1$, $s = 1$ and $d = 1$.
This metric doesn't consider the fact that the added operations can be performed in parallel.

\end{document}